\documentclass{article}
\usepackage{ijcai20}

\usepackage{times}
\usepackage{soul}
\usepackage{url}
\usepackage[hidelinks]{hyperref}
\usepackage[utf8]{inputenc}
\usepackage[small]{caption}
\usepackage{graphicx}
\usepackage{amsmath}
\usepackage{amsthm}
\usepackage{booktabs}
\usepackage{algorithm}
\usepackage{comment}
\usepackage{algorithmic}
\usepackage{boxedminipage}
\usepackage{graphicx}
\usepackage{microtype}
\usepackage{amsmath}
\usepackage{booktabs}
\usepackage{algorithm}
\usepackage{algorithmic}
\usepackage{bbm}
\usepackage{misc}
\usepackage{pifont}
\usepackage{xspace}
\usepackage{url}
\usepackage[normalem]{ulem}

\usepackage{amssymb}
\usepackage{amsmath}

\usepackage{thmtools,thm-restate}

\usepackage{mathtools}

\usepackage{stmaryrd} 

\usepackage{tikz}
\usetikzlibrary{decorations.pathreplacing}
\usetikzlibrary{matrix}

\newtheorem{lemma}{Lemma} 
\newtheorem{theorem}{Theorem} 
 
\newtheorem{defn}{Definition}
\newtheorem{exmp}{Example}

\title{A Journey into Ontology Approximation: From Non-Horn to Horn}

\author{
Anneke Haga$^1$
\and
Carsten Lutz$^1$\and
Johannes Marti$^{2}$\And
Frank Wolter$^3$
\affiliations
$^1$Universit\"at Bremen, Germany\\
$^2$Universiteit van Amsterdam, Netherlands\\
$^3$University of Liverpool, UK\\
\emails
\{anneke, clu\}@uni-bremen.de,
johannes.marti@gmail.com,
wolter@liverpool.ac.uk
}

\begin{document}

\maketitle

\begin{abstract}
  We study complete approximations of an ontology formulated in a
  non-Horn description logic (DL) such as \ALC in a Horn DL such
  as~\EL. We provide concrete approximation schemes that are
  necessarily infinite and observe that in the $\ELU$-to-$\EL$ case
  finite approximations tend to exist in practice and are guaranteed to
  exist when the source ontology is acyclic. In contrast, neither of
  these are the case for $\ELU_\bot$-to-$\EL_\bot$ and for
  $\ALC$-to-$\EL_\bot$ approximations. We also define a notion of
  approximation tailored towards ontology-mediated querying, connect
  it to subsumption-based approximations, and identify a case where
  finite approximations are guaranteed to exist.
\end{abstract}

\section{Introduction}

Despite prominent standardization efforts such as OWL, a large variety
of description logics (DLs) continues to be used as ontology
languages. In fact, ontology designers choose a DL suitable for their
purposes based on many factors including expressive power,
computational properties, and tool support
\cite{DBLP:books/daglib/0041477}. Since ontology engineering
frequently involves (partial) reuse of existing ontologies, this
raises the problem of converting an ontology written in some source DL
$\Lmc_S$ into a desired target DL~$\Lmc_T$. A particularly important
case is ontology approximation where $\Lmc_T$ is a fragment of
$\Lmc_S$, studied for example in
\cite{DBLP:conf/aaai/PanT07,DBLP:conf/aaai/RenPZ10,DBLP:conf/aimsa/BotoevaCR10,DBLP:conf/cade/MartinezFGHH14,DBLP:journals/jair/ZhouGNKH15,DBLP:conf/ijcai/BotcherLW19}.

In practice, ontology approximation is often done in an ad hoc way by
dropping all statements from the source ontology $\Omc_S$ that are not
expressible in $\Lmc_T$, or at least the inexpressible parts of such
statements. It is well-known that this results in incomplete
approximations, that is, there will be knowledge in $\Omc_S$ that
could be expressed in $\Lmc_T$, but is not contained in the resulting
approximated ontology. The degree and nature of the resulting
incompleteness is typically neither understood nor analyzed. One
reason for this unsatisfactory situation might be the fact that it is
by no means easy to construct complete approximations and, even worse,
finite complete approximations are not guaranteed to exist. This was
studied in depth in \cite{DBLP:conf/ijcai/BotcherLW19} where
ontologies formulated in expressive Horn DLs such as
Horn-$\mathcal{SHIF}$ and $\mathcal{ELI}$ are approximated in
tractable Horn DLs such as \EL. For example, it is shown there that
finite complete $\ELI$-to-\EL approximations do not exist even in
extremely simple cases including those occurring in
practice. The authors then lay out a new
research program for ontology approximation that consists in mapping
out the structure of complete (infinite) ontology approximations as a
tool for guiding informed decisions when constructing incomplete
(finite) approximations in practice, and also to enable a better
understanding of the degree and nature of incompleteness.

In this paper, we consider $\Lmc_S$-to-$\Lmc_T$ ontology approximation
where $\Lmc_S$ is a non-Horn DL such as \ALC and $\Lmc_T$ is a
tractable Horn DL such as \EL. Arguably, these are extremely natural
cases of ontology approximation given that Horn vs.\ non-Horn is
nowadays the most important classification criterion for DLs
\cite{DBLP:books/daglib/0041477}. Non-Horn DLs include expressive
features such as negation and disjunction and require `reasoning by
cases' which is computationally costly, but also have considerably
higher expressive power than Horn DLs. Horn DLs, in contrast, enjoy
favourable properties such as the existence of universal models and of
`consequence-based' reasoning algorithms that avoid reasoning by cases
\cite{DBLP:conf/birthday/CucalaGH19}. Despite being natural, however,
non-Horn-to-Horn approximation turns out to be a challenging
endeavour.
%


We start with the fundamental case of $\ELU$-to-\EL
approximation. Given an \ELU ontology $\Omc_S$, we aim to find a
(potentially infinite) \EL ontology $\Omc_T$ such that for all \EL
concepts $C,D$ in the signature of $\Omc_S$, $\Omc_S \models C
\sqsubseteq D$ iff $\Omc_T \models C \sqsubseteq D$. 
\begin{exmp}
\label{ex:0}
  Consider the $\ELU$ ontology 
$$
  \begin{array}{r@{}r@{\;}c@{\;}l@{}l}
    \Omc_S = \{ & \mn{Job} &\sqsubseteq& \mn{MainJob} \sqcup \mn{SideJob} \\[0.5mm]
&  \exists \mn{job} . \mn{SideJob} &\sqsubseteq& \exists \mn{job}
  . (\mn{MainJob} \sqcap \mn{PartTime}) \ \}.
\end{array}
$$
Then the following is an $\EL$
approximation of $\Omc_S$: 
$$
  \begin{array}{@{}r@{\;}c@{\;}l@{}l}
    \Omc_T = \{ \quad  \exists \mn{job} . \mn{SideJob} &\sqsubseteq& \exists \mn{job}
  . (\mn{MainJob} \sqcap \mn{PartTime}) \\[0.5mm]
 \exists \mn{job} . \mn{Job}
&\sqsubseteq& 
\exists \mn{job} . \mn{MainJob} \\[0.5mm]
 \exists \mn{job} . (\mn{Job} \sqcap \mn{PartTime}) 
&\sqsubseteq& 
\exists \mn{job} . (\mn{MainJob} \sqcap \mn{PartTime}) \ \}. 
\end{array}
$$
\end{exmp}
The last two lines of $\Omc_T$ illustrate that \EL consequences of \ELU
ontologies can be rather non-obvious. 

We first prove that finite approximations need not exist in the
$\ELU$-to-\EL case and that depth bounded approximations may be
non-elementary in size. Our main result is then a concrete
approximation scheme that makes explicit the structure of complete
infinite approximations and aims to keep as much structure of the
source ontology as possible. 
An interesting and, given the
results in \cite{DBLP:conf/ijcai/BotcherLW19}, surprising feature of
our scheme is that it can be expected to often deliver \emph{finite}
approximations in practical cases. We perform a case study based on the
Manchester ontology corpus that confirm this expectation.  We also
show that if $\Omc_S$ is an acyclic \ELU ontology, then a finite \EL
approximation always exists (though it need not be acyclic). The finite
approximations that we obtain are too large to be directly used in
practice. Nevertheless, we view our results as positive and
believe that in practice approximations of reasonable size 
often exist, as in Example~\ref{ex:0}.  A `push button
technology' for constructing them, however, is outside of the scope of
this paper.

We then proceed to the cases of $\ELU_\bot$-to-$\EL_\bot$ and
$\ALC$-to-$\EL_\bot$ approximations which turn out to be closely
related to each other. They also turn out to be significantly
different from the $\ELU$-to-\EL case in that finite approximations do
not exist in extremely simple (and practical) cases, much like in the
Horn approximation cases studied in
\cite{DBLP:conf/ijcai/BotcherLW19}.  Also, finite approximations of
acyclic ontologies are no longer guaranteed to exist.  While this is
not good news, it is remarkable that the addition of the
$\bot$ symbol has such a dramatic effect. We again provide an
(infinite) approximation scheme.

Finally, we propose a notion of approximation that is tailored towards
applications in ontology-mediated querying
\cite{DBLP:conf/rweb/CalvaneseGLLPRR09}
 and show that 
 it is intimately related to the subsumption-based approximations that
 we had studied before.  Remarkably, if we concentrate on atomic
 queries (AQs), then we obtain finite approximations even in the
 $\ALC$-to-$\EL_\bot$ case.  Compared to the related work presented in
 \cite{DBLP:journals/ai/KaminskiNG16}, we do not require the
 preservation of all query answers, but only of a maximal subset
 thereof, and our method is applicable to all ontologies formulated in
 the source DL chosen rather than to a syntactically restricted class. We
 also observe an interesting application to the rewritability of
 ontology-mediated queries.


All proofs are deferred to the appendix.

\section{Preliminaries}
\label{sect:prelims}

Let $\NC$ and $\NR$ be disjoint and
countably infinite sets of \emph{concept names} and \emph{role
  names}. In the description logic \ALC, \emph{concepts} $C,D$ are
built according to the syntax rule
$$
   C,D ::= \top \mid \bot \mid A \mid \neg C \mid C \sqcap D \mid
   C \sqcup D \mid \exists r . C \mid \forall r . C
$$
where $A$ ranges over \NC and $r$ over \NR.   
The \emph{depth} of a concept is the nesting depth of the constructors
$\exists r$ and $\forall r$ in it. For example, the concept $\exists r
. B \sqcap \exists r . \exists s .A$ is of depth~2.  We introduce
other DLs as fragments of \ALC. An \emph{$\ELU_\bot$ concept} is an
\ALC concept that does not contain negations $\neg C$ and value
restrictions $\forall r . C$. An \emph{$\EL_\bot$ concept} is an
$\ELU_\bot$ concept that does not contain disjunctions $C \sqcup
D$. \emph{\ELU concepts} and \emph{\EL concepts} are defined likewise,
but additionally forbid the use of the bottom concept $\bot$.

For any of these DLs \Lmc, an \emph{\Lmc ontology} is a
set of \emph{concept inclusions (CIs)} $C \sqsubseteq D$ where $C$
and $D$ are \Lmc concepts. While ontologies used in practice have to
be finite, we frequently consider also infinite
ontologies. W.l.o.g., we assume that all occurrences of
$\bot$ in $\ELU_\bot$ ontologies are in CIs of the form $C \sqsubseteq
\bot$, where $C$ does not contain $\bot$. An \emph{acyclic ontology}
\Omc is a set of concept inclusions $A \sqsubseteq C$ and
\emph{concept equivalences} $A \equiv C$ where $A$ is a concept name
(that is, it is not a compound concept), the left-hand sides are
unique, and \Omc does not contain a definitiorial cycle $A_0
\bowtie_1 C_0, \dots, A_n \bowtie_n C_n$, $\bowtie_i \in \{
{\sqsubseteq},{\equiv}\}$, where $C_i$ contains $A_{i+1 \,
  \mn{mod} \, n+1}$ for all $i \leq n$. An equivalence $A \equiv C$
can be viewed as two CIs $A \sqsubseteq C$ and $C
\sqsubseteq A$ and thus every acyclic ontology is an ontology in the
original sense.\looseness=-1  

A \emph{signature} $\Sigma$ is a set of concept and role names,
uniformly referred to as \emph{symbols}. We use $\mn{sig}(X)$ to
denote the set of symbols used in any syntactic object $X$ such as a
concept or an ontology.  If $\mn{sig}(X) \subseteq \Sigma$, we also
say that $X$ is \emph{over} $\Sigma$.  The \emph{size} of a (finite)
syntactic object $X$, denoted $||X||$, is the number of symbols needed
to write it, with every occurrence of a concept and role name contributing one.

The semantics of concepts and ontologies is defined in terms of
\emph{interpretations} $\Imc=(\Delta^\Imc,\cdot^\Imc)$ as usual,
see~\cite{DBLP:books/daglib/0041477}.
%
An interpretation \Imc \emph{satisfies} a CI $C \sqsubseteq D$ if
$C^\Imc \subseteq D^\Imc$, an equivalence $A \equiv C$ if $A^\Imc =
C^\Imc$, and it is a \emph{model} of an ontology \Omc if it satisfies all
CIs in~\Omc. 
Concept $C$ is \emph{subsumed by} concept $D$ w.r.t.\ ontology \Omc,
written $\Omc \models C \sqsubseteq D$, if every model \Imc of \Omc
satisfies the CI $C \sqsubseteq D$; we then also say that the CI is a
\emph{consequence} of \Omc.  Subsumption can be decided in polynomial
time in $\EL_\bot$ and is \ExpTime-complete between $\ELU$ and
$\ALC$~\cite{DBLP:books/daglib/0041477}.  We now give our main
definition of approximation. With concept of depth bounded by
$\omega$, we mean concepts of unrestricted depth.
\begin{defn} \label{def:approx_types} Let $\Omc_S$ be an \ALC
  ontology, $\text{sig}(\Omc_S) = \Sigma$, $\Lmc_T$ any of the DLs
  introduced above, and $\ell \in \mathbb{N} \cup \{ \omega \}$.
  A (potentially infinite) $\Lmc_T$ ontology $\Omc_T$ is an
  \emph{$\ell$-bounded $\Lmc_T$ approximation of $\Omc_S$} if
  $$\Omc_S \models C \sqsubseteq D \text{ iff } \Omc_T \models
    C \sqsubseteq D
  $$
  for all $\Lmc_T$ concepts $C,D$ over $\Sigma$ of depth bounded by
  $\ell$.  $\Omc_T$ is \emph{non-projective} if $\mn{sig}(\Omc_T)
  \subseteq \Sigma$ and \emph{projective} otherwise. We refer to
  $\omega$-bounded $\Lmc_T$ approximations as \emph{$\Lmc_T$
    approximations}.
\end{defn}
We refer to the ``if'' direction of the biimplication in
Definition~\ref{def:approx_types} as \emph{soundness} of the
approximation and to the ``only if'' direction as \emph{completeness}.
Infinite 
approximations
always exist: take as $\Omc_T$ the set of all \Lmc CIs $C \sqsubseteq
D$ with $C,D$ over $\Sigma$ and $\Omc_S \models C \sqsubseteq D$.
In the same way, finite (non-projective) depth-bounded approximations
always exist.
With \emph{$\Lmc_S$-to-$\Lmc_T$ approximation}, $\Lmc_S$ a DL and
$\Lmc_T$ a fragment of $\Lmc_S$, we mean the task to approximate an
$\Lmc_S$ ontology in~$\Lmc_T$, possibly using an infinite ontology. 



\section{\ELU-to-\EL Approximation}
\label{sect:elutoel}

We consider \ELU-to-\EL approximation as the simplest case of
approximating non-Horn ontologies in a Horn DL. 


\paragraph{Fundamentals.} 
 We start with observing that projective approximations are more
powerful than non-projective ones.
\begin{restatable}{prop}{firstprop}
\label{ex:1}
  The \ELU ontology
  %
  $$
  \begin{array}{r@{}rcl@{}l}
    \Omc_S = \{ & A &\sqsubseteq& B_1 \sqcup B_2,\\[0.5mm]
& \exists r . B_i &\sqsubseteq& B_i, & \text{for } i \in \{1,2\}\\[0.5mm]
&    B_i \sqcap A' &\sqsubseteq& M& \text{for } i \in \{1,2\} \ \}.
  \end{array}
  $$
  %
  has a finite projective \EL approximation, but every non-projective 
  \EL approximation is infinite. 
\end{restatable}
In fact, a finite projective \EL approximation $\Omc_T$ of the
ontology $\Omc_S$ from Proposition~\ref{ex:1} is obtained from
$\Omc_S$ by replacing the CI in the first line with 
  %
  $$
    A \sqsubseteq X_{B_1 \sqcup B_2}, \ 
     \exists r . X_{B_1 \sqcup B_2} \sqsubseteq X_{B_1 \sqcup
      B_2},\ 
  X_{B_1 \sqcup B_2} \sqcap A' \sqsubseteq M.
  $$
  The intuitive reason for why $\Omc_S$ has no finite non-projective \EL
  approximation is that $\Omc_S \models A' \sqcap
  \exists r^{n} . A \sqsubseteq M$ for all $n \geq
  0$. 
%
  Proposition~\ref{ex:1} indicates that projective approximations are
  preferable. 
  Since they also seem perfectly acceptable from an application
  viewpoint, we concentrate on the projective case and from now
  on mean projective approximations whenever we speak of
  approximations.

  To illustrate the challenges of \ELU-to-\EL approximation, it is
  instructive to consider a candidate approximation scheme that might
  be suggested by Proposition~\ref{ex:1}.
  We use $\mn{sub}(\Omc_S)$ to denote
  the set of all subconcepts of (concepts in) the ontology $\Omc_S$
  and $\mn{sub}^-(\Omc_S)$ to denote the restriction of $\mn{sub}(\Omc_S)$
  to concept names and existential restrictions $\exists r . C$.  
%
\begin{figure}
  \centering 
   \begin{boxedminipage}{\columnwidth}
     \vspace*{-2mm}
  $$
  \begin{array}{rcll}
      C &\sqsubseteq& X_C \\[0.5mm]
      X_{D_1} \sqcap C &\sqsubseteq& X_{D_2} &
        \text{ if } \Omc_S \models D_1 \sqcap C \sqsubseteq D_2 \\[0.5mm]
      X_{D_1} \sqcap X_{D_2} &\sqsubseteq& X_{D_3} &
        \text{ if } \Omc_S \models D_1 \sqcap D_2 \sqsubseteq D_3 \\[0.5mm]
      \exists r . X_{D_1} &\sqsubseteq& X_{D_2} &
         \text{ if } \Omc_S \models \exists r . D_1 \sqsubseteq D_2 \\[0.5mm]
      X_{D_1} &\sqsubseteq& \exists r . X_{D_2} &
         \text{ if } \Omc_S \models D_1 \sqsubseteq \exists r . D_2 \\[0.5mm]
      X_{D_1} &\sqsubseteq& C &
         \text{ if } \Omc_S \models D_1 \sqsubseteq C \\[0.5mm]
  \end{array}
  $$
  \end{boxedminipage}
  \caption{Candidate \EL approximation $\Omc_T$.}
\label{fig:first}
\vspace*{-3mm}
\end{figure}
%
%
%
We use $\mn{Con}(\Omc_S)$ to denote the set of all  non-empty
conjunctions of concepts from $\mn{sub}^-(\Omc_S)$ without
repetitions
%
%
and $\mn{Dis}(\Omc_S)$ to mean the set of all 
disjunctions
of concepts from $\mn{Con}(\Omc_S)$ without repetitions.
%
Now, a (finite projective) candidate \EL approximation scheme is given
in Figure~\ref{fig:first} where $C$ ranges over $\mn{sub}(\Omc_S)$ and
$D_1,D_2,D_3$ range over $\mn{Dis}(\Omc_S)$.  It indeed yields an
approximation when applied to the ontology $\Omc_S$ in
Proposition~\ref{ex:1}. There are, however, two major
problems. First, the syntactic structure of $\Omc_S$ is lost
completely, which is undesirable in practice where ontologies are the
result of a careful modeling effort. We could include all \EL concept
inclusions from $\Omc_S$ in the approximation, but this would be
purely cosmetic since all such CIs are already implied.
%
%
Second, the approximation is incomplete in general. In fact, finite 
approximations need not exist also in the projective case 
while the
approximation scheme in Figure~\ref{fig:first} is always finite.  
\begin{restatable}{prop}{extwoprop}
\label{ex:2}
The \ELU ontology
%
  $$
  \begin{array}{rrcl}
  \Omc_S = \{ & A &\sqsubseteq& B_1 \sqcup B_2, \\[0.5mm]
  & \exists r . B_2 &\sqsubseteq& \exists r . (B_1 \sqcap L), \\[0.5mm]
  & L &\sqsubseteq&\exists s . L \qquad\qquad \}
  \end{array} 
 $$
 has no finite 
\EL approximation.
\end{restatable}
%
The intuitive
reason for why $\Omc_S$ has no finite \EL approximation is
that
$
\Omc_S \models \exists r.(A \sqcap \exists s^n . \top) \sqsubseteq \exists r.(B_1 \sqcap \exists s^n.\top) 
$
for all $n \geq 0$. 

The ontology in Proposition~\ref{ex:2} can be varied to show that even
bounded depth approximations can get very large. The function
$\mn{tower}: \mathbb{N} \times \mathbb{N} \rightarrow \mathbb{N}$ is
defined as $\mn{tower}(0,n) := n$ and
$\mn{tower}(k+1,n):=2^{\mn{tower}(k,n)}$.
\begin{restatable}{prop}{propnonelem}
\label{prop:nonelem}
  Let $\Omc^{n}_S$ be obtained from the ontology $\Omc_{S}$ in
  Proposition~\ref{ex:2} by replacing the bottommost CI with 
  $$
L 
  \sqsubseteq A_1 \sqcap \hat{A}_1 
  \sqcap \cdots \sqcap A_n \sqcap \hat A_n \sqcap \exists r_1.L \sqcap \exists r_2 .L
 $$
 %
 Then for all $n,\ell \geq 1$ and any $\ell$-bounded 
 $\mathcal{EL}$ approximation $\Omc_T$ of $\Omc^n_S$,
 $||\Omc_T|| \geq \mn{tower}(\ell,n)$.
\end{restatable}
\paragraph{A Complete Approximation.} 
We present a more careful approximation scheme that aims to preserve
the structure of $\Omc_S$, is complete,  and yields a finite
approximation in many practical cases.
Let $\Omc_S$ be an \ELU ontology to be approximated.  As a
preliminary, we assume that for all CIs $C \sqsubseteq D \in \Omc_S$,
$C$ is an \EL concept. If this is not the case, then we can
rewrite $\Omc_S$ by exhaustively replacing every disjunction $C \sqcup
D$ that occurs (possibly as a subconcept) on the left-hand side of a
concept inclusion in $\Omc_S$ with a fresh concept name $X_{C \sqcup
  D}$ and adding the inclusions $C \sqsubseteq X_{C \sqcup D}$ and $D
\sqsubseteq X_{C \sqcup D}$. It is not hard to see that the resulting
ontology $\Omc'_S$ is of size polynomial in $||\Omc_S||$ and a
conservative extension of $\Omc_S$ in the sense that $\Omc_S \models C
\sqsubseteq D$ iff $\Omc'_S \models C \sqsubseteq D$ for all \ELU
concepts $C,D$ over $\mn{sig}(\Omc_S)$.  Consequently, every \EL
approximation of $\Omc'_S$ is also a projective \EL approximation
of~$\Omc_S$ and we can work with $\Omc'_S$ in place of~$\Omc_S$.

Let $\ell \in \mathbb{N} \cup \{ \omega \}$. The proposed \EL
approximation $\Omc^\ell_T$ of $\Omc_S$ is given in
Figure~\ref{fig:second} where $D_1,D_2$ range over $\mn{Dis}(\Omc_S)$
and $D$ ranges over $\mn{Dis}^-(\Omc_S)$, the set of all disjunctions
in $\mn{Dis}(\Omc_S)$ that have at least two disjuncts. 
\begin{figure}[t]
   \begin{boxedminipage}{\columnwidth}
     \begin{center}
     \vspace*{-2mm} 
  $$
  \begin{array}{rcll}
    C &\sqsubseteq& \mn{DNF}(E)^\uparrow & \text{if } C \sqsubseteq E
  \in \Omc_S \\[0.5mm]
  X_{D} \sqcap D_1^\uparrow &\sqsubseteq&
  D_2^\uparrow & \text{if } \Omc_S \models D \sqcap D_1 \sqsubseteq
  D_2\\[0.5mm]
 \exists r . X_{D} &\sqsubseteq& D_1^\uparrow & \text{if }
   \Omc_S \models \exists r . D \sqsubseteq D_1 \\[0.5mm]
   F^\uparrow &\sqsubseteq& \exists r . G & \text{if } \Omc_S \models F \sqsubseteq
  \exists r . G 
  \end{array}
  $$
  \end{center}
  where in the last line
  \begin{itemize}

  \item $F$ is an \EL concept over $\mn{sig}(\Omc_S)$ decorated with 
    disjunctions from $\mn{Dis}(\Omc_S)$ at leaves and

  \item $G$ is an     $\Omc_S$-generatable \EL concept over $\mn{sig}(\Omc_S)$

  \end{itemize}
  such that $\mn{depth}(F) \leq \mn{depth}(G) < \ell$.
  \end{boxedminipage}
  \caption{$\ell$-bounded \EL approximation $\Omc^\ell_T$.}
\label{fig:second}
\vspace*{-3mm}
\end{figure}
We still have to define the notation and terminology used in the
figure.  For an \ELU concept $C$ such that all disjunctions in $C$ are
from $\mn{Dis}(\Omc_S)$, we use $C^\uparrow$ to denote the \EL concept
obtained from $C$ by replacing every outermost
$D \in \mn{Dis}^-(\Omc_S)$ with a fresh concept name $X_D$. Set
$\mn{DNF}(C)=C$ if $C$ is a concept name or of the form
$\exists r . D$,
$\mn{DNF}(C_1 \sqcap C_2)=\mn{DNF}(C_1) \sqcap \mn{DNF}(C_2)$, and
define $\mn{DNF}(C_1 \sqcup C_2)$ to be the \ELU-concept obtained by
converting $C_1 \sqcup C_2$ into disjunctive normal form (DNF),
treating existential restrictions $\exists r. D$ as atomic concepts,
that is, the argument $D$ is not modified. Note that while
$||\mn{DNF}(C)||$ may be exponential in $||C||$, we have
$||\mn{DNF}(C)^\uparrow|| \leq ||C||$.  By \emph{decorating an \EL
  concept $C$ with disjunctions from $\mn{Dis}(\Omc_S)$ at leaves}, we
mean to replace subconcepts $\exists r . E$ of~$C$ with $E$ of depth~0 by
$\exists r. (E \sqcap D)$, $D \in \mn{Dis}(\Omc_S)$. As a special
case, we can replace $C$ with $C \sqcap D$, $D \in \mn{Dis}(\Omc_S)$,
if $C$ is of depth~0.  \looseness=-1
\begin{defn}
  
An \EL concept $C$ is
\emph{$\Omc_S$-generatable} if there is an $\exists r . D \in
\mn{sub}(\Omc_S)$ that occurs on the right-hand side of a CI in
$\Omc_S$ and satisfies $\Omc_S \models D \sqsubseteq C$.
\end{defn}

Let us explain the proposed approximation.  The first three lines of
Figure~\ref{fig:second} can be viewed as a more careful version of the
first four lines of Figure~\ref{fig:first}. In the first line, we
preserve the structure of $\Omc_S$ as long as it lies outside the
scope of a disjunction operator, thanks to the careful definition of
$\mn{DNF}(C)$.  This is not cosmetic as in the candidate approximation
in Figure~\ref{fig:first}: since we introduce the concept names $X_D$
only when a disjunction is `derived' (first line) and only for
disjunctions $D \in \mn{Dis}^-(\Omc_S)$, $\Omc^\ell_T$ is no longer
guaranteed to be an approximation when the first line in
Figure~\ref{fig:second} is dropped.
The last line of the approximation addresses the effect
illustrated by Proposition~\ref{ex:2}. It is strong enough so that a
counterpart of the second last line in Figure~\ref{fig:first} is not
needed.
An example application of our approximation scheme is given in
arxive version: put the example right here.

%
An interesting aspect of our approximation is that it turns out to be
finite in many practical cases. In fact, 
it is easy to see that $\Omc^\ell_T$ is
finite for all $\ell < \omega$
and that $\Omc^\omega_T$ is finite if and only if there are only
finitely many \EL concepts that are $\Omc_S$-generatable, up to
logical equivalence; we then say that $\Omc_S$ is \emph{finitely
  generating}. Since ontologies from practice tend to have a simple
structure, one might expect that they often enjoy this
property. Below, we report about a case study that confirms this
expectation.  

How does the approximation scheme in Figure~\ref{fig:second} relate to
the examples given above?  For the ontologies $\Omc_S$ in
Example~\ref{ex:0} and in Proposition~\ref{ex:1}, our approximation
$\Omc^\omega_T$ contains all CIs in the approximation $\Omc_T$ given
in place.  Of course, $\Omc^\omega_T$ also contains a lot of
additional CIs that, however, do not result in any new consequences $C
\sqsubseteq D$ with $C, D$ \EL concepts over $\mn{sig}(\Omc_S)$. It
seems very difficult to identify up front those CIs that are really
needed. We can remove them after constructing $\Omc^\omega_T$ by
repeatedly deciding conservative extensions
\cite{DBLP:journals/jsc/LutzW10}, but this is not practical given the
size of $\Omc^\omega_T$. Nevertheless, both ontologies $\Omc_S$ are
finitely generating and thus in both cases $\Omc^\omega_T$ is finite. In
Example~\ref{ex:0}, the $\Omc_S$-generatable concepts are $\top$,
$\mn{MainJob}$, $\mn{PartTime}$, and $\mn{MainJob} \sqcap \mn{PartTime}$
(up to logical equivalence) while there are no $\Omc_S$-generatable
concepts for Proposition~\ref{ex:1}. For Proposition~\ref{ex:2}, there
are infinitely many $\Omc_S$-generatable concepts such as $\exists s^n
. \top$ for all $n \geq 0$.






\paragraph{Examples.} 
To illustrate the proposed approximation scheme, we pick up some of
the previous examples again.

Recall that the \ELU ontology $\Omc_S$ from Example~\ref{ex:0} has
a finite \EL approximation $\Omc_T$, given in place, and repeated here
for the reader's convenience:
$$
  \begin{array}{@{}r@{\;}c@{\;}l@{}l}
    \Omc_S = \{  \hspace*{17mm} \mn{Job} &\sqsubseteq& \mn{MainJob} \sqcup \mn{SideJob} \\[0.5mm]
  \exists \mn{job} . \mn{SideJob} &\sqsubseteq& \exists \mn{job}
  . (\mn{MainJob} \sqcap \mn{PartTime}) \ \} \\[4mm]
    \Omc_T = \{ \  \exists \mn{job} . \mn{SideJob} &\sqsubseteq& \exists \mn{job}
  . (\mn{MainJob} \sqcap \mn{PartTime}) \\[0.5mm]
 \exists \mn{job} . \mn{Job}
&\sqsubseteq& 
\exists \mn{job} . \mn{MainJob} \\[0.5mm]
 \exists \mn{job} . (\mn{Job} \sqcap \mn{PartTime}) 
&\sqsubseteq& 
\exists \mn{job} . (\mn{MainJob} \sqcap \mn{PartTime}) \ \}. 
\end{array}
$$
The first CI in $\Omc_T$ is directly taken over from $\Omc_S$, via the
first line of the approximation scheme in Figure~\ref{fig:second}.
The second and third CI in $\Omc_T$ are instances of the fourth line
in Figure~\ref{fig:second} since, as already noted, $\mn{MainJob}$ and
$\mn{MainJob} \sqcap \mn{PartTime}$ are both $\Omc_S$-generatable.

Of course, the approximation scheme in Figure~\ref{fig:second}
introduces many additional CIs that, however, are not needed in this
particular case for the approximation to be complete. Let us still
consider a few of them. The first line creates
$$
  \mn{Job} \sqsubseteq X_{\mn{MainJob} \sqcup \mn{SideJob}}.
$$
Note that we do not need the DNF conversion from the first line of
Figure~\ref{fig:second} here since $\mn{MainJob} \sqcup \mn{SideJob}$ is
already in DNF. In fact, we only need this conversion if the
right-hand side of a CI in $\Omc_S$ contains a disjunction nested
inside a conjunction nested inside a disjunction, which should be rare
in practice. Then, for example, Line~2 of Figure~\ref{fig:second}
yields uninteresting CIs such as 
$$
X_{\mn{MainJob} \sqcup \mn{SideJob}} \sqsubseteq \mn{MainJob} \sqcup
\mn{SideJob}
$$
which clearly do not add new knowledge.  Line~3
of Figure~\ref{fig:second} also does not yield any interesting CIs
$$
  \exists \mn{job} . X_{\mn{MainJob} \sqcup \mn{SideJob}} \sqsubseteq
  D^\uparrow
$$
as every disjunction $D \in \mn{Dis}(\Omc_S)$ with $\Omc _S \models
\exists \mn{job} . \mn{MainJob} \sqcup \mn{SideJob} \sqsubseteq D$
is tautological. Additional concept names $X_D$ are also introduced,
e.g.\ via Line~2 and the~CI
$$
X_{\mn{MainJob} \sqcup \mn{SideJob}} \sqcap \mn{Job} \sqsubseteq
X_{(\mn{MainJob} \sqcap \mn{Job}) \sqcup (\mn{SideJob} \sqcap \mn{Job})},
$$
triggering new applications of the second and third line in turn, but
no new knowledge in $\mn{sig}(\Omc_S)$ is ever derived.

\medskip

To see non-redundant applications of Lines~2 and~3 of
Figure~\ref{fig:second}, reconsider the ontology $\Omc_S$ given in
Proposition~\ref{ex:1} and its approximation given below that
proposition, here again repeated for convenience:
  $$
  \begin{array}{r@{}rcl@{}l}
    \Omc_S = \{ & A &\sqsubseteq& B_1 \sqcup B_2,\\[0.5mm]
& \exists r . B_i &\sqsubseteq& B_i, & \text{for } i \in \{1,2\}\\[0.5mm]
&    B_i \sqcap A' &\sqsubseteq& M& \text{for } i \in \{1,2\} \ \}.
  \end{array}
  $$
  and $\Omc_T$ is obtained from $\Omc_S$ by replacing the CI in the
  first line with
  %
  $$
    A \sqsubseteq X_{B_1 \sqcup B_2}, \ 
     \exists r . X_{B_1 \sqcup B_2} \sqsubseteq X_{B_1 \sqcup
      B_2},\ 
  X_{B_1 \sqcup B_2} \sqcap A' \sqsubseteq M.
  $$
  The first additional CI is an instance of Line~1 of
  Figure~\ref{fig:second}, the second CI is an instance of Line~3, and
  the third CI is an instance of Line~2. As already mentioned, there
  are no $\Omc_S$-generatable concepts since $\Omc_S$ does not have
  existential restrictions on the right-hand side of CIs, and thus
  Line~4 of Figure~\ref{fig:second} cannot be applied. If applied
  naively, Line~2 and~3 yield many additional CIs and introduce many
  additional concepts $X_D$, but as in Example~\ref{ex:0}, they
  do not derive any new knowledge in $\mn{sig}(\Omc_S)$.
  
\paragraph{Case Study.} 
We have considered the seven 
non-trivial \ELU ontologies that are part of the Manchester OWL
corpus.\footnote{http://owl.cs.manchester.ac.uk/publications/supporting-material/owlcorpus/}
The size of the ontologies ranges from 113 to 813 concept inclusions
and equalities. All ontologies use disjunction on the right-hand side
of CIs (thus in a non-trivial way) and none of them is acyclic. We
have been able to prove that all these ontologies are finitely
generating and thus the approximation $\Omc^\omega_T$ is finite. Our
proof relies on the following observation.
\begin{restatable}{lemma}{lemCharOgen}
\label{lem:charOgen}
$\Omc_S$ is not finitely generating iff for every $n \geq 0$, there is an
$\exists r . D \in \mn{sub}(\Omc_S)$ that occurs on the right-hand
side of a CI and a sequence $r_1,\dots,r_n$ of role names from
$\Omc_S$ such that $\Omc_S \models D \sqsubseteq \exists r_1 . \cdots
. \exists r_n . \top$.
\end{restatable}
In our implementation, we use role inclusions to avoid going through
all of the exponentially many sequences $r_1,\dots,r_n$. 
Lemma~\ref{lem:charOgen} can also be used to show the following.
\begin{restatable}{theorem}{thmdecidable}
\label{thm:decidable}
It is decidable whether a given \ELU-ontology $\Omc_S$ is finitely generating.
\end{restatable}
By what was said above, this implies that it is decidable whether the
approximation $\Omc^\omega_T$ from Figure~\ref{fig:second} is finite.

\paragraph{Soundness and Completeness.}
We now establish soundness and completeness of the proposed approximation,
the main result in this section.
\begin{theorem}
\label{thm:corr} 
For every $\ell \in \mathbb{N} \cup \{ \omega \}$, $\Omc^\ell_T$ is an  
$\ell$-bounded \EL approximation of $\Omc_S$.  
\end{theorem}
While soundness is easy to show, completeness is remarkably
 subtle to prove. 
%
 It is stated by the following lemma which shows that our
 approximation $\Omc^\ell_T$ is actually stronger than required in
 that it preserves all \EL subsumptions $C \sqsubseteq D$ with $D$ of
 depth bounded by $\ell$ and $C$ of unrestricted depth.
\begin{restatable}{lemma}{lemcomp} \label{lem:comp} Let $\ell \in
  \mathbb{N} \cup \{ \omega \}$. Then $\Omc_S \models C_0 \sqsubseteq
  D_0$ implies $\Omc^\ell_T \models C_0 \sqsubseteq D_0$ for all \EL
  concepts $C_0,D_0$ over $\mn{sig}(\Omc_S)$ such that the role depth
  of $D_0$ is bounded by $\ell$.
\end{restatable}
%
%
The proof of Lemma~\ref{lem:comp} is the most substantial one in this
paper. It uses a chase procedure for \ELU ontologies that is
specifically tailored towards proving completeness in that it is
deterministic rather than disjunctive and mimics the concept
inclusions in Figure~\ref{fig:second}. Showing that this chase is 
complete is far from trivial. 

\paragraph{Fewer Symbols.}
%
The number of fresh concept names $X_D$ in $\Omc^\ell_T$ is double
exponential in $||\Omc_S||$ since the number of disjunctions in
$\mn{Dis}^-(\Omc)$ is. However, $\Omc^\ell_T$ can be rewritten
into an ontology $\widehat\Omc^\ell_T$ that uses only single
exponentially many fresh concept names and is still an $\ell$-bounded
approximation of~$\Omc_S$. The idea is to transition from disjunctive
normal form to conjunctive normal form, that is, to replace each
concept name $X_D$, $D \in \mn{Dis}^-(\Omc)$, with a conjunction of
concept names $Y_{D'}$ where $D'$ is a disjunction of concepts from
$\mn{sub}^-(\Omc)$, rather than conjunctions thereof.
Details are in the appendix. 
%
%
%
\begin{theorem}
\label{thm:smallercomplete}
For every $\ell \in \mathbb{N} \cup \{ \omega \}$,
  $\widehat \Omc^\ell_T$ is an $\ell$-bounded \EL approximation of $\Omc_S$.  
\end{theorem}

\paragraph{Acyclic Ontologies.}
%
Using Lemma~\ref{lem:charOgen}, one can show that $\Omc^\omega_T$ is
finite whenever $\Omc_S$ is an acyclic \ELU ontology. In fact, the
length $n$ of role sequences with the properties stated in the lemma
is bounded by $||\Omc_S||$ if $\Omc_S$ is acyclic. 
\begin{theorem}
  Every acyclic \ELU ontology has a finite \EL approximation.
\end{theorem}
There is, however,
more that we can say about acyclic ontologies.  We first observe that
there are acyclic \ELU ontologies that have finite \EL approximations,
but no \EL approximation that is an acyclic ontology.
\begin{exmp}
  Consider the acyclic \ELU ontology 
  $$\Omc_{S}=\{A\equiv (B_{1}\sqcap
  B_{2}) \sqcup (B_{1} \sqcap B_{3})\}.$$ Then
$
\Omc_{T}=\{ B_{1}\sqcap B_{2}\sqsubseteq A, B_{1}\sqcap B_{3}\sqsubseteq A, A \sqsubseteq B_{1}\}
$
is an $\mathcal{EL}$ approximation of $\Omc_{S}$, but $\Omc_S$ has no 
 \EL approximation that is an acyclic ontology, finite or infinite.
\end{exmp}
Further, our approximations $\Omc^\ell_T$ can be simplified for
acyclic \ELU ontologies $\Omc_S$. Let $\widetilde\Omc^\ell_T$ be
defined like $\Omc^\ell_T$ in Figure~\ref{fig:second}, except that in
the last line, $F$ ranges only over concept names (not decorated with
disjunctions) rather than over compound concepts, a significant
simplification.
\begin{theorem}
\label{thm:acyclicbetter} 
Let $\ell \in \mathbb{N} \cup \{ \omega \}$ and let $\Omc_S$ be an
acyclic \ELU ontology. Then $\widetilde\Omc^\ell_T$ is an
$\ell$-bounded \EL approximation of $\Omc_S$.
\end{theorem}
Based on this observation, constructing finite \EL approximations of
acyclic \ELU ontologies does not seem infeasible in practice.


\section{\ALC-to-$\EL_\bot$ Approximation}
\label{sec:bottom}

We consider $\ELU_\bot\!$-to-$\EL_\bot$ and $\ALC$-to-$\EL_\bot$
approximation which turn out to be closely related to each other and
significantly different from $\ELU$-to-$\EL$ approximation.

It immediately follows from the results in Section~\ref{sect:elutoel}
that finite approximations are guaranteed to exist neither in the
$\ELU_\bot\!$-to-$\EL_\bot$ nor in the $\ALC$-to-$\EL_\bot$
case. However, while we have argued that finite
$\ELU$-to-$\EL$ approximations can be expected to exist in many
practical cases, this does not appear to be true for
$\ELU_\bot\!$-to-$\EL_\bot$ and $\ALC$-to-$\EL_\bot$. The following
example illustrates the problem.
%
  \begin{exmp}
\label{ex:ELUbot}
     Consider the $\ELU_\bot$ ontology
  $$
  \begin{array}{rr@{\;}c@{\;}l}
  \Omc_S = \{&  
  A_1 &\sqsubseteq& M \sqcup N_1, \\
  & A_2 &\sqsubseteq& M \sqcup N_2, \\
  & \exists r . N_1 \sqcap \exists r . N_2 &\sqsubseteq& \bot \qquad\qquad\qquad \}.
  \end{array} 
 $$
There are no $\Omc_S$-generatable \EL concepts. Yet, there
 is no finite $\EL_\bot$ approximation of $\Omc_S$. Informally, this
 is because $$\Omc_S \models \exists r . (A_1 \sqcap \exists r^n
 . \top) \sqcap \exists r . (A_2 \sqcap \exists r^n . \top)
 \sqsubseteq \exists r . (M \sqcap \exists r^n . \top)$$ for all $n
 \geq 1$.\footnote{A
 formal proof is analogous to that of Proposition~\ref{ex:2}.}
\end{exmp}
While the above example is for $\ELU_\bot\!$-to-$\EL_\bot$, there is an
additional effect in $\ALC$-to-$\EL_\bot$ that already occurs
for very simple ontologies $\Omc_S$.
%
\begin{exmp}
\label{ex:ELI}
The $\mathcal{ALC}$ ontology $\Omc_{S}=\{A \sqsubseteq \forall r.B\}$
has no finite $\EL_{\bot}$ approximation.  This is shown in
\cite{DBLP:conf/ijcai/BotcherLW19} for the equivalent \ELI ontology
$\{ \exists r^- . A \sqsubseteq B\}$. Informally, this is because
$\Omc_S \models A \sqcap \exists r^{n+1}.\top \sqsubseteq \exists r.(B
\sqcap \exists r^n . \top) $ for all $n \geq 1$.
\end{exmp}
Note that the ontology $\Omc_S$ in Example~\ref{ex:ELI} is acyclic and
thus in contrast to the $\ELU$-to-$\EL$ case, finite $\EL_\bot$
approximations of acyclic \ALC ontologies need not
exist. In a sense, Example~\ref{ex:ELUbot} shows the same negative
result for the $\ELU_\bot\!$-to-$\EL_\bot$ case. While the ontology
used there is not strictly acyclic, acyclic ontologies do not make
much sense in the case of $\ELU_\bot$ and additionally admitting CIs
$C_1 \sqcap C_2 \sqsubseteq \bot$ as used in Example~\ref{ex:ELUbot}
seems to be the most modest extension possible that incorporates
$\bot$ in a meaningful way.

Despite these additional challenges, we can extend the approximation
given in Section~\ref{sect:elutoel} to $\ELU_\bot\!$-to-$\EL_\bot$ and
to $\ALC$-to-$\EL_\bot$ when we are willing to drop
$\Omc_S$-generatability and, as a consequence, accept the fact that
approximations are infinite unless they are depth bounded. Note that
the latter is also the case in \Lmc-to-\EL approximation where \Lmc
is an expressive Horn DL such as
\ELI~\cite{DBLP:conf/ijcai/BotcherLW19}.

We first reduce $\ALC$-to-$\EL_\bot$ approximations to
$\ELU_\bot\!$-to-$\EL_\bot$ approximations. Let $\Omc_S$ be an \ALC
ontology. We can transform $\Omc_S$ into an $\ELU_\bot$ ontology as
follows:
\begin{enumerate}

\item replace each subconcept $\forall r . C$ with $\neg \exists r
  . \neg C$;

\item select a concept $\neg C$ such that $C$ contains no negation,
  replace all occurrences of $\neg C$ with the fresh concept name
  $A_{\neg C}$, and add the CIs $\top \sqsubseteq C \sqcup A_{\neg
    C}$ and $C \sqcap A_{\neg C} \sqsubseteq \bot$; repeat until no
  longer possible.

\end{enumerate}
The resulting ontology $\Omc'_S$ is of size polynomial in $||\Omc_S||$
and a conservative extension of $\Omc_S$ in the sense that $\Omc_S
\models C \sqsubseteq D$ iff $\Omc'_S \models C \sqsubseteq D$ for all
\ALC concepts $C,D$ over $\mn{sig}(\Omc_S)$.  Consequently, every
$\EL_\bot$ approximation of $\Omc'_S$ is also a (projective) $\EL_\bot$
approximation of~$\Omc_S$.

\begin{figure}[t]
   \begin{boxedminipage}{\columnwidth}
     \vspace*{-2mm} 
  \begin{center}
  $$
  \begin{array}{rcll}
    C &\sqsubseteq& \mn{DNF}(E)^\uparrow & \text{if } C \sqsubseteq E
  \in \Omc_S \\[0.5mm]
  X_{D} \sqcap D_1^\uparrow &\sqsubseteq&
  D_2^\uparrow & \text{if } \Omc_S \models D \sqcap D_1 \sqsubseteq
  D_2\\[0.5mm]
 \exists r . X_{D} &\sqsubseteq& D_1^\uparrow & \text{if }
   \Omc_S \models \exists r . D \sqsubseteq D_1 \\[0.5mm]
   F^\uparrow &\sqsubseteq& \exists r . G & \text{if } \Omc_S \models F \sqsubseteq
  \exists r . G 
  \end{array}
  $$
  \end{center}
  where in the last line $F$ is an \EL concept over $\mn{sig}(\Omc_S)$
  decorated with disjunctions from $\mn{Dis}(\Omc_S)$ at leaves and
  $G$ is an \EL concept over $\mn{sig}(\Omc_S)$ such that 
    \begin{enumerate}

    \item $F$ has no top-level conjunct $\exists r . F'$ s.t.\ $\Omc_S
      \models F' \sqsubseteq G$;




    \item   $\mn{depth}(F) \leq \mn{depth}(G) < \ell$.

    \end{enumerate}
  \end{boxedminipage}
  \caption{$\ell$-bounded $\EL_\bot$ approximation $\Omc^\ell_T$.}
\label{fig:third}
\vspace*{-3mm}
\end{figure}
It thus suffices to consider $\ELU_\bot\!$-to-$\EL_\bot$
approximations. Thus let $\Omc_S$ be an $\ELU_\bot$ ontology. For each
$\ell \in \mathbb{N} \cup \{ \omega \}$, the $\EL_\bot$ approximation
$\Omc^\ell_T$ of $\Omc_S$ is given in Figure~\ref{fig:second} where
again $D$ ranges over $\mn{Dis}^-(\Omc_S)$ and $D_1,D_2$ range over
$\mn{Dis}(\Omc_S)$; both $\mn{Dis}(\Omc_S)$ and $\mn{Dis}^-(\Omc_S)$
are defined exactly as for $\ELU$ ontologies and in $\mn{DNF}(C)$ we
drop all disjuncts that contain $\bot$ as a conjunct, possibly
resulting in the empty disjunction (which represents $\bot$).  Point~1
can be viewed as an optimization that sometimes helps to avoid the
expensive last line.  There, a \emph{top-level conjunct} means a
concept $F_i$ if $F$ takes the form $F_1 \sqcap \cdots \sqcap F_n$, $n
\geq 1$.
In the appendix
we point out another non-trivial such optimization.
\begin{theorem}
\label{thm:corrbot}
  $\Omc^\ell_T$ is an $\ell$-bounded $\EL_\bot$ approximation of $\Omc_S$.  
\end{theorem}
The proof of Theorems~\ref{thm:corr} and~\ref{thm:corrbot} also
establishes another result that will turn out to be interesting in the
context of ontology-mediated queries in Section~\ref{sect:OMQs}. We
use $\Omc_T^-$ to denote the restriction of $\Omc^\omega_T$ to the
(instantiations) of the first three lines in Figure~\ref{fig:third}
(equivalently: Figure~\ref{fig:second}). Clearly, $\Omc_T^-$ is always
finite.
%
\begin{restatable}{theorem}{thmcorrminus}
\label{thm:corrminus}
  Let $C_0,D_0$ be $\EL_\bot$ concepts with $D_0 \in
  \mn{sub}(\Omc_S)$. Then $\Omc_S \models C_0 \sqsubseteq D_0$ iff
  $\Omc^-_T \models C_0 \sqsubseteq D_0$.
\end{restatable}

%

\enlargethispage*{4mm}

\section{Approximations and Query Evaluation}
\label{sect:OMQs}

The notion of approximations given in Section~\ref{sect:prelims} is
tailored towards preserving subsumptions.
In ontology-mediated querying, in contrast, the main aim of
approximation is to preserve as many query answers as possible. We
propose a suitable notion of approximation and show that the results
obtained in the previous sections have interesting applications also in
ontology-mediated querying. 
%

Let \NI be a countably infinite set of \emph{individual names}
disjoint from \NC and \NR.  An \emph{ABox} is a finite set of
\emph{concept assertions} $A(a)$ and \emph{role assertions} $r(a,b)$
where $A \in \NC$, $r \in \NR$, and $a,b \in \NI$. We use
$\mn{Ind}(\Amc)$ to denote the set of individual names 
in
the ABox \Amc. An interpretation \Imc \emph{satisfies} a concept
assertion $A(a)$ if $a \in A^\Imc$ and a role assertion $r(a,b)$ if
$(a,b) \in r^\Imc$. 
It is a \emph{model} of an ABox if it satisfies all assertions in it. A
\emph{$\Sigma$-ABox} is an ABox \Amc with $\mn{sig}(\Amc)=\Sigma$.

An \emph{ontology-mediated query (OMQ)} is a triple
$Q=(\Omc,\Sigma,q)$ with \Omc an ontology, $\Sigma \subseteq
\mn{sig}(\Omc) \cup \mn{sig}(q)$ an ABox signature,
and $q$ an actual query. While conjunctive queries (CQs) and unions of
CQs are a popular choice for formulating $q$ and our central
Definition~\ref{def:approx_queryNEW} below makes sense also for these
richer query languages, for simplicity we concentrate on \emph{atomic
  queries (AQs)} $A(x)$ where $A$ is a concept name and on \emph{\EL
  queries (ELQs)} $C(x)$ where $C$ an \EL concept. We also
mention \emph{\ALC queries (ALCQs)} $C(x)$ where $C$ is an \ALC
concept.  Note that all such queries are unary. We use ELQ$(\Sigma)$
to denote the language of all ELQs that use only symbols from
signature~$\Sigma$. Let $(\Lmc,\Qmc)$ denote the \emph{OMQ language}
that contains all OMQs~$Q$ in which $\Omc$ is formulated in DL \Lmc
and $q$ in query language \Qmc, such as in
$(\EL,\text{AQ})$.

Let $Q=(\Omc,\Sigma,C(x))$ be an OMQ and \Amc a $\Sigma$-ABox. Then $a \in
\mn{Ind}(\Amc)$ is an \emph{answer} to $Q$ on \Amc, written $\Amc
\models Q(a)$, if $a \in C^\Imc$ for all models \Imc of \Omc and \Amc.
For OMQs $Q_1$ and $Q_2$, $Q_i =(\Omc_i,\Sigma,q_i)$, we say that
$Q_1$ is \emph{contained} in $Q_2$ and write $Q_1 \subseteq Q_2$ if
for every $\Sigma$-ABox \Amc and $a \in \mn{Ind}(\Amc)$, $\Amc
\models Q_1(a)$ implies $\Amc \models Q_2(a)$. We say that $Q_1$ is
\emph{equivalent} to $Q_2$ and write $Q_1 \equiv Q_2$ if $Q_1
\subseteq Q_2$ and $Q_2 \subseteq Q_1$. 

A natural definition of ontology approximation in the context of OMQs
is as follows.
%
%
%
\begin{defn} \label{def:approx_queryNEW} Let $\Omc_S$ be an \ALC
  ontology, 
$\Lmc_T$ one of the DLs
 from Section~\ref{sect:prelims}, and $\mathcal{Q}$ a query
  language.  An $\Lmc_T$ ontology $\Omc_T$ is an \emph{$\Lmc_T$
    approximation of $\Omc_S$ w.r.t $\mathcal{Q}$} if for all queries $q \in
  \Qmc$ and all signatures $\Sigma$ with $\Sigma \cap \mn{sig}(\Omc_T)
  \subseteq \mn{sig}(\Omc_S)$,
\begin{enumerate}

\item $(\Omc_S,\Sigma,q) \supseteq (\Omc_T,\Sigma,q)$ and

  \item $(\Omc_S,\Sigma,q) \supseteq Q$ implies
    $(\Omc_T ,\Sigma,q) \supseteq Q$ for all OMQs $Q
    =(\Omc'_T,\Sigma,q)$ with $\Omc'_T  \in
    \Lmc_T$.

\end{enumerate}
\end{defn}
$\Omc_T$ might use fresh symbols and thus approximations
are projective.
Informally, Point~1 is a soundness condition and Point~2 formalizes
`to preserve as many query answers as possible'.
%
It is not guaranteed
that the OMQs $(\Omc_S,\Sigma,q)$ and
$(\Omc_T,\Sigma,q)$ are equivalent for all relevant queries $q$ and
signatures $\Sigma$, and the following example shows that this is
in fact impossible to achieve. 
\begin{exmp}
  Let $\Omc_S$ be the \ELU ontology
  $$
  \Omc_S = \{  \top \sqsubseteq B_1 \sqcup B_2 \} \cup \{ B_i \sqcap
  \exists r . B_i \sqsubseteq A \mid i \in \{1,2\} \}
 $$
 Then an \EL approximation of $\Omc_S$ w.r.t.\
 ELQ is 
  $$
  \begin{array}{r@{\;}c@{\;}l}
  \Omc_T &=& \{ B_1 \sqcap B_2 \sqcap \exists r . \top \sqsubseteq A, \
  \exists r . (B_1 \sqcap B_2) \sqsubseteq A  \}\\[0.5mm]
&&\cup \; \{ B_i \sqcap
  \exists r . B_i \sqsubseteq A \mid i \in \{1,2\} \}.
  \end{array} 
 $$
 However, there is no OMQ in $(\EL,\text{ELQ})$ that is equivalent to
 $(\Omc_S,\{r\},A(x))$ since it would have to return $a$ as an answer
 on the ABox $\{ r(a,a) \}$, but not on the ABox $\{
 r(a,b),r(b,a)\}$. No OMQ from $(\EL,\text{ELQ})$ has this property.
\end{exmp}
It turns out that the approximations from Sections~\ref{sect:elutoel}
and~\ref{sec:bottom} are also useful in the context of
Definition~\ref{def:approx_queryNEW} when we choose ELQ or AQ as the
query language. In particular, it follows from Theorem~\ref{thm:corrminus}
that every \ALC ontology $\Omc_S$ has a finite $\EL_\bot$
approximation w.r.t.\ AQ.
%
\begin{restatable}{theorem}{thmOMQresults}
\label{thm:OMQresults}
Let $\Omc_S$ be an \ALC ontology, $\mn{sig}(\Omc_S)=\Sigma$.
Then
   \begin{enumerate}

   \item the ontology $\Omc_T^\omega$ from
     Section~\ref{sec:bottom} is an $\EL_\bot$ approximation of
     $\Omc_S$ w.r.t.~ELQ$(\Sigma)$;

     \item the ontology $\Omc_T^-$ from Section~\ref{sec:bottom} is a
     (finite) $\EL_\bot$ approximation of $\Omc_S$ w.r.t.~AQ;

   \item if $\Omc_S$ falls within \ELU, then the ontology
     $\Omc_T^\omega$ from Section~\ref{sect:elutoel} is an $\EL$
     approximation of $\Omc_S$ w.r.t.~ELQ$(\Sigma)$. 

  \end{enumerate}
\end{restatable}
Point~2 also implies that $\Omc_T^-$ is an \EL approximation of
$\Omc_S$ w.r.t.~AQ whenever $\Omc_S$ is an \ELU ontology.  We close
with an interesting application of Theorem~\ref{thm:OMQresults}.

The topic of rewriting an OMQ into a simpler query language has
received a lot of interest in the literature, see for example
\cite{DBLP:journals/jar/CalvaneseGLLR07,DBLP:journals/ai/GottlobKKPSZ14,DBLP:journals/ai/KaminskiNG16,DBLP:journals/lmcs/FeierKL19}.
An OMQ $Q$ is $(\Lmc,\Qmc)$-\emph{rewritable} if there is an OMQ $Q'$
in the OMQ language $(\Lmc,\Qmc)$ such that $Q \equiv Q'$. 

By virtue
of Theorem~\ref{thm:OMQresults}, we can decide whether an OMQ
$Q=(\Omc,\Sigma,A(x))$ from $(\ALC,\text{AQ})$ is
$(\EL_\bot,\text{AQ})$-rewritable. It can be seen that this is the
case if and only if $Q$ is equivalent to an OMQ $Q' \in
(\EL_\bot,\text{AQ})$ 
of the form
$(\Omc',\Sigma,A(x))$. By
Condition~2 of
Definition~\ref{def:approx_queryNEW}, it thus suffices to
construct the finite $\EL_\bot$ approximation
$\Omc_T^-$ of $\Omc$ w.r.t.\ AQ from Theorem~\ref{thm:OMQresults}
and check whether $Q \equiv
(\Omc_T^-,\Sigma,A(x))$, which is decidable
\cite{DBLP:journals/tods/BienvenuCLW14}.  This
result extends to $(\ALC,\text{ALCQ})$ since every OMQ from this language is
equivalent to one from $(\ALC,\text{AQ})$.  Via the results in
\cite{DBLP:conf/ijcai/FeierLW18}, this can be lifted further to
a certain class of conjunctive queries.\looseness=-1  
\begin{theorem}
  Given an OMQ $Q \in (\ALC,\text{ALCQ})$, it is decidable whether $Q$
  is $(\EL_\bot,\text{AQ})$-rewritable.
\end{theorem}

\section{Conclusion}

We have investigated the structure and finiteness of ontology
approximations when transitioning from non-Horn DLs to Horn DLs.  We
believe that our results shed significant light on the situation.
It remains, however, an important and challenging topic for future
work to push our techniques further towards practical
applicability. Also, there are many other relevant cases of
approximation. As a first step, one might think about extending the
DLs considered in this paper with role inclusions.  It might further
be interesting to study the problem to decide whether a given (finite)
candidate is an approximation of a given ontology. We expect this to
be quite non-trivial. A related result in \cite{DBLP:conf/kr/LutzSW12}
states that it is between \ExpTime and 2\ExpTime to decide whether a
given \ELU ontology $\Omc_S$ of a restricted syntactic form has a finite
complete \EL approximation. Without the restriction, even decidability
is open.



\section*{Acknowledgements}
Supported by the DFG Collaborative Research Center 1320 EASE - Everyday Activity Science and Engineering.




\cleardoublepage

\appendix

\section{Proofs for Propositions~\ref{ex:1}, \ref{ex:2},
  and~\ref{prop:nonelem}}

We state the results to be proved again.

\firstprop*
\noindent
\begin{proof}\
We show that $\Omc_{S}$ has no finite non-projective \EL approximation. Observe that the ontology
$\Omc$ obtained from $\Omc_{S}$ by replacing the topmost CI with the infinite set 
$$
M = \{A' \sqcap \exists r^{n}.A \sqsubseteq M \mid n\geq 0 \}
$$
is an infinite non-projective \EL approximation of $\Omc_{S}$. Now assume for a proof by
contradiction that there exists a finite non-projective \EL approximation of $\Omc_{S}$. Then,
by compactness of reasoning in \EL, there exists a finite subset $\Omc'$ of $\Omc$ 
that is an \EL approximation of $\Omc_{S}$.
Let $n$ be maximal such that $A' \sqcap \exists r^{n}.A \sqsubseteq M\in \Omc'$. Then
$\Omc'\not\models A'\sqcap \exists r^{n+1}.A\sqsubseteq M$ and we have derived a contradiction.
\end{proof}

To prove Proposition~\ref{ex:2} and~\ref{prop:nonelem}, we use the
following lemma from~\cite{DBLP:journals/jsc/LutzW10}.  If $C$ is an
\EL concept of the form $C_1 \sqcap \cdots \sqcap C_n$, $n \geq 1$,
then the \emph{top-level conjuncts} of $C$ are $C_1,\dots,C_n$.
\begin{lemma}\label{lem:extra}
Let $\Omc$ be an $\mathcal{EL}$ ontology and $C,D$ be \EL concepts.
Then $\Omc\models C \sqsubseteq \exists r.D$ implies that 
\begin{enumerate}
\item there exists a top-level conjunct $\exists r.C'$ of $C$ such that $\Omc\models C'\sqsubseteq D$ or
\item there exists a $C' \in \mn{sub}(\Omc)$ such that $\Omc\models C \sqsubseteq \exists r.C'$ and 
$\Omc\models C'\sqsubseteq D$.
\end{enumerate}
\end{lemma}

\extwoprop*
\noindent
\begin{proof}\ Let $\Omc_T$ be a (potentially projective)
  $\mathcal{EL}$ approximation of $\Omc_S$. Then for all $n\geq 0$ and
  $m>n$, we have
\begin{itemize}

\item[(a)] $\Omc_T \models \exists r.(A \sqcap \exists s^{n}. \top)\sqsubseteq
\exists r. (B_{1} \sqcap \exists s^{n}.\top)$ and

\item[(b)] $\Omc_T \not\models \exists r.(A \sqcap \exists s^{n}. \top)\sqsubseteq
\exists r. (B_{1} \sqcap \exists s^{m}.\top)$ 

\end{itemize}
since the same is true for $\Omc_S$. To establish the desired result,
it suffices to argue that for every $n \geq 0$, there is a $C_n \in
\mn{sub}(\Omc_T)$ such that $\Omc_T \models C_n \sqsubseteq
B_{1}\sqcap \exists s^{n}.\top$ and $\Omc_T \not\models C_n
\sqsubseteq B_{1} \sqcap \exists s^{m}.\top$ for any $m > n$.  In
fact, if this
is the case, then $\Omc_T$ has infinitely many subconcepts
and is thus infinite.

Let $n \geq 0$. First note that 
\begin{itemize}

\item[(c)] $\Omc_T \not\models A \sqcap \exists s^{n}.\top
  \sqsubseteq B_{1}$.

\end{itemize}
because the same is true for $\Omc_S$. It follows from $(a)$, $(c)$,
and Lemma~\ref{lem:extra} that there exists a $C \in \mn{sub}(\Omc_T)$
such that $\Omc_T \models \exists r . (A \sqcap \exists s^{n}.\top)
\sqsubseteq \exists r . C$ and $\Omc_T \models C \sqsubseteq B_{1}
\sqcap \exists s^{n}.\top$.  Set $C_n = C$. By choice and by (b), $C_n$ 
is as desired.
\end{proof}

\propnonelem
\noindent
\begin{proof}\
Assume that a depth bound $\ell\geq 1$ is given.
\EL concepts $C_{1},C_{2}$ are \emph{incomparable} w.r.t.\ $\Omc_S$ if
neither $\Omc_S\models C_{1}\sqsubseteq C_{2}$ nor $\Omc_S\models
C_{2}\sqsubseteq C_{1}$.  Take a set $\Omega$ of \EL concepts of depth
bounded by $\ell-1$ that are pairwise incomparable w.r.t.\ $\Omc_S$ and use
only the symbols
$r_{1},r_{2},A_{1},\hat{A}_{1},\ldots,A_{n},\hat{A}_{n}$.
It is straightforward to
construct such a set $\Omega$ and that 
has size at least $\mn{tower}(\ell,n)$. 
%
It then suffices to show that for every $E\in \Omega$ there exists a
$C_{E} \in \mn{sub}(\Omc_T)$ such that $\Omc_T\models
C_{E}\sqsubseteq E$ and $\Omc_T\not\models C_{E}\sqsubseteq E'$ for
any $E'\in \Omega$ with $E'\not=E$. 

Let $E\in \Omega$. Then
\begin{itemize}
\item[(a)]
$\Omc_T \models \exists r.(A \sqcap E) \sqsubseteq \exists r.(B_{1}
\sqcap E)$, 

\item[(b)]
$\Omc_T \not\models \exists r.(A \sqcap E) \sqsubseteq \exists r.(B_{1}
\sqcap E')$ for any $E' \in \Omega$ with $E' \neq E$, and

\item[(c)] $\Omc_T\not\models A \sqcap E \sqsubseteq B_{1}$. 

\end{itemize}
since the same is true for $\Omc_S$.  Thus, similarly to the proof of
Proposition~\ref{ex:2} we can show that must exist a 
$C \in \mn{sub}(\Omc_T)$ such that $\Omc_T\models \exists r.(A\sqcap E)
\sqsubseteq \exists r.C$ and $\Omc_T \models C \sqsubseteq
B_{1} \sqcap E$ and use $C$ as $C_E$.
\end{proof}

\section{Proof of Theorem~\ref{thm:decidable}}

\lemCharOgen*

\noindent 
\begin{proof}\
Observe that the number of non-logically equivalent \EL concepts over $\Sigma=\mn{sig}(\Omc_{S})$ and of depth bounded by $n$
is finite, for any natural number $n\geq 0$. Moreover, any two \EL concepts of distinct depth are not logically equivalent. 
Thus, there are infinitely many non-logically equivalent $\Omc_{S}$-generatable \EL concepts if, and only if, for every 
$n\geq 0$ there exists an $\Omc_{S}$-generatable \EL concept of depth $n$. The latter holds if, and only if, for
every $n\geq 0$ there exist role names $r_{1},\ldots,r_{n}$ in $\Omc_{S}$ such that $\exists r_{1}.\cdots \exists r_{n}.\top$
is $\Omc_{S}$-generatable.
\end{proof}

\thmdecidable*

\noindent 
\begin{proof}\
It follows from Lemma~\ref{lem:charOgen} that
it suffices to decide whether there exists a bound $\ell\geq 0$ such that for every 
$\exists r.D\in \mn{sub}(\Omc_{S})$ on the right hand side of a CI in $\Omc_{S}$ and any sequence 
$r_{1},\ldots,r_{n}$ of role names in $\Omc_{S}$, if
$\Omc_{S}\models D \sqsubseteq \exists r_{1}.\cdots \exists r_{n}.\top$, then $n\leq \ell$. 
We show that there exists such an $\ell$ if, and only if, there exists such an $\ell$ with 
$\ell\leq |\mn{sig}(\Omc_{S})|\times 2^{2^{||\Omc_{S}||}}$. Then decidability follows directly.
We use a straightforward pumping argument to show the claim. Assume that there are 
$n> |\mn{sig}(\Omc_{S})|\times 2^{2^{||\Omc_{S}||}}$, 
$\exists r.D\in \mn{sub}(\Omc_{S})$ on the right hand side of a CI in $\Omc_{S}$, and role names $r_{1},\ldots,r_{n}$ in $\Omc_{S}$
with $\Omc_{S}\models D \sqsubseteq \exists r_{1}.\cdots \exists r_{n}.\top$. We show that then there exists
such a concept $\exists r.D$ and sequence of role names of length $n'>n$.
An \emph{$\Omc_{S}$-type} is a subset $t$ of the closure under single negation
of $\mn{sub}(\Omc_{S})$ such that for any $C\in\mn{sub}(\Omc_{S})$ either $C\in t$ or $\neg C\in t$ and there exists
a model $\Imc$ of $\Omc_{S}$ and $d\in \Delta^{\Imc}$ with $d\in (\bigsqcap_{C\in t}C)^{\Imc}$. We identify an $\Omc_{S}$-type
$t$ with the concept $\bigsqcap_{C\in t}C$ and let $D(\Omc_{S})$ be the set of disjunctions of $\Omc_{S}$-types 
(without repetitions).
We show that there exists a sequence $X_{1},\ldots,X_{n}\in D(\Omc_{S})$ such that 
$$
\Omc_{S}\models D \sqsubseteq \exists r_{1}.X_{1}, \quad \Omc_{S}\models X_{i}\sqsubseteq \exists r_{i+1}.X_{i+1},
$$
for all $i<n$. The proof is as follows. Let $X_{1}$ be the set of all $\Omc_{S}$-types $t$ such that
there exist a model $\Imc$ of $\Omc_{S}$ and $d,e\in\Delta^{\Imc}$ with $(d,e)\in r_{1}^{\Imc}$, $d\in D^{\Imc}$
and $e\in t^{\Imc}$. Assume that $X_{i}$ has been defined. Then $X_{i+1}$ is the set of all $\Omc_{S}$-types $t$ such that
there exist a model $\Imc$ of $\Omc_{S}$ and $d,e\in\Delta^{\Imc}$ with $(d,e)\in r_{i+1}^{\Imc}$, $d\in X_{i}^{\Imc}$
and $e\in t^{\Imc}$. By definition
$$
\Omc_{S}\models D \sqsubseteq \forall r_{1}.X_{1}, \quad \Omc_{S}\models X_{i}\sqsubseteq \forall r_{i+1}.X_{i+1},
$$
and now one can readily show by induction on $i$, and
using that $\Omc_{S}\models  D \sqsubseteq \exists r_{1}.\cdots \exists r_{n}.\top$, that
$$
\Omc_{S}\models D \sqsubseteq \exists r_{1}.X_{1}, \quad \Omc_{S}\models X_{i}\sqsubseteq \exists r_{i+1}.X_{i+1},
$$
for all $i<n$. Thus, as $n> |\mn{sig}(\Omc_{S})|\times 2^{2^{||\Omc_{S}||}}$, there exist $1< i < j \leq n$ such that 
$r_{i}=r_{j}$ and $X_{i}=X_{j}$. But then 
$$
\Omc_{S}\models D \sqsubseteq \exists r_{1}.\cdots \exists r_{j-1}.\exists r_{i}.\cdots.\exists r_{n}.\top,
$$
and we have found the sequence of role names of length $n+(j-i)>n$ we wanted.
\end{proof}

\section{Proof of Theorem~\ref{thm:corr}}

\subsection{Preliminaries}
\label{app:prelims}

We write $\Amc \models
C(a)$ if $a \in C^\Imc$ where \Imc is \Amc viewed as an interpretation
in the obvious way.  An ABox \Amc is \emph{ditree-shaped} if the
directed graph $G_\Amc=(\mn{Ind}(\Amc),\{(a,b) \mid r(a,b) \in \Amc
\})$ is a tree and there are no multi-edges, that is, $r(a,b),s(a,b)
\in \Amc$ implies $r=s$. Every \EL concept $C$ can be viewed as 
a ditree-shaped ABox $\Amc_C$ in an obvious way. 

\medskip

We will sometimes also use \emph{extended ABoxes}, that is, ABoxes
\Amc that can also contain concept assertions of the form~$C(a)$, $C$
a compound concept. If all concepts that occur in such assertions are
formulated in a description logic \Lmc, we speak of \emph{extended
  \Lmc-ABoxes}.  If \Amc is an extended ABox, then we use $\Amc^-$
to denote the non-extended ABox obtained from \Amc by removing all
assertions $C(a)$ where $C$ is not a concept name.

\medskip

We next introduce a standard chase procedure for \EL ontologies. The
procedure uses ABoxes as a data structure. Let \Omc be an \EL
ontology. There is a single chase rule that can be applied to an ABox
\Amc:
\begin{itemize}

\item if $C \sqsubseteq D \in \Omc$ and $\Amc \models C(a)$, then a
  copy of $\Amc_D$ whose individuals are disjoint from those in \Amc
  and replace \Amc with the union of \Amc and $\Amc_D$, identifying
  the root of the latter with $a$.
  
\end{itemize}
The chase starts with an ABox $\Amc_0$ and exhaustively applies the
above rule in a fair way, resulting in sequence of ABoxes
$\Amc_0,\Amc_1,\dots$. The \emph{result} of the chase is the
(potentially infinite) ABox $\bigcup_{i \geq 0} \Amc_i$ obtained in
the limit, denoted $\mn{ch}_\Omc(\Amc)$. The result is unique since
the chase is oblivious, that is, a rule can applied to $C \sqsubseteq
D$ and $C(a)$ even if $\Amc \models D(a)$ already holds. A proof of
the following is standard and omitted.
\begin{lemma}
\label{lem:chasecorr}
$\Omc \models C \sqsubseteq D$ iff $\mn{ch}_\Omc(\Amc_C) \models
D(a_0)$, for all \EL concepts $C$ and $D$.
\end{lemma}

\subsection{Main Proof}

We start with soundness.
%
\begin{restatable}{lemma}{lemsound} \label{lem:sound} $\Omc^\omega_T
  \models C_0 \sqsubseteq D_0$ implies $\Omc_S \models C_0 \sqsubseteq
  D_0$ for all \EL concepts $C_0,D_0$ over $\mn{sig}(\Omc_S)$.
\end{restatable}
\noindent
\begin{proof}\ Assume that $\Omc^\omega_T \models C_0 \sqsubseteq D_0$
  where $C_0,D_0$ are \EL concepts over $\mn{sig}(\Omc_S)$. Then
  $\mn{ch}_{\Omc^\omega_T}(\Amc_{C_0}) \models D_0(a_0)$, $a_0$ the
  root of $\Amc_{C_0}$. Let $\Amc_{C_0}=\Amc_0,\Amc_1,\dots$ be a
  sequence of ABoxes produced by the \EL chase of $\Amc_{C_0}$ with
  $\Omc^\omega_T$. Clearly, all ABoxes $\Amc_0,\Amc_1,\dots$ are
  ditree-shaped and can thus be viewed as an \EL concept $C_i$. For an
  \EL concept $C$ over $\mn{sig}(\Omc^\omega_T)$, let $C^{\downarrow}$
  be the \ELU concept obtained from $C$ by replacing every $X_D$ with
  $D \in \mn{Dis}(\Omc_S)$.  We prove the following by induction on
  $i$.
  \\[2mm]
  {\bf Claim.} $\Omc_S \models C^\downarrow_i \sqsubseteq
  C^\downarrow_{i+1}$ for all $i \geq 0$.
  \\[2mm]
  To prove the claim, let $i \geq 0$. 
  $\Amc_{i+1}$ was obtained from $\Amc_{i}$ by applying the
  chase rule. Thus let $C \sqsubseteq D \in \Omc^\omega_T$,
  $\Amc_{i} \models C(a)$, and let $\Amc_{i+1}$ be obtained from
  $\Amc_{i}$ by taking a copy of $\Amc_D$ whose individuals are
  disjoint from those in $\Amc_{i}$ and defining $\Amc_{i+1}$ as the
  union of $\Amc_{i}$ and $\Amc_D$, identifying the root of the
  latter with~$a$. 
  By definition of $\Omc^\omega_T$, we have $\Omc_S \models
  C^\downarrow \sqsubseteq D^\downarrow$. 
  By construction of $\Amc_{i+1}$, we thus have $\Omc_S \models
  C_i^\downarrow \sqsubseteq C_{i+1}^\downarrow$ as required and thus
  the claim is proved.

  \smallskip


  From $\mn{ch}_{\Omc^\omega_T}(\Amc_{C_0}) \models D_0$, we obtain
  $\Amc_i \models D_0$ for some~$i$. Since $D_0$ is over
  $\mn{sig}(\Omc_S)$, $\Amc_i \models D_0$ implies $\emptyset \models
  C^\downarrow_i \sqsubseteq D_0$. Together with the claim and since
  $C_0=C_0^\downarrow$, this gives $\Omc_S \models C_0 \sqsubseteq D_0$.
\end{proof}
We now address completeness, starting with the essential
Lemma~\ref{lem:compaux} below. Preparing for the case of
$\ELU_\bot\!$-to-$\EL_\bot$ approximations, we state and 
prove the lemma directly for this case. This requires a few
preliminaries.

\smallskip

Let $\Omc$ be an $\ELU_\bot$ ontology. For every \EL concept $C$, we
define $\mn{Dis}_{\Omc}(C)$ as in the case without $\bot$.  This can
now be the empty disjunction, which we identify with $\bot$.  In fact,
$C$ is satisfiable w.r.t.\ $\Omc$ if and only if $\mn{Dis}_{\Omc}(C) =
\bot$. We set $\bot^\uparrow=\bot$. We further 
associate with every
\EL concept $C$ a disjunction
$\mn{Dis}^\EL_{\Omc}(C)$ 
that contains a disjunct $\bigsqcap S$ for every set $S \subseteq
\mn{sub}^-(\Omc)$ such that there is a model \Imc of $\Omc$ and a
$d \in C^\Imc$ with
$$S=\{ E \in \mn{sub}^-(\Omc) \mid d \in
E^\Imc \text{ and $E$ is an \EL concept} \}$$
while this is not true for any
proper subset of $S$. If $\mn{Dis}^\EL_{\Omc}(C)$ consists
of a single disjunct that is the empty conjunction, we identify it
with $\top$. The empty disjunction is again identified with $\bot$.

\smallskip For the following lemma, we assume that $\Omc_S$ is an
$\ELU_\bot$ ontology. The lemma refers to $\Omc^-_T$. Note that 
when $\Omc_S$ is formulated in $\ELU$, then $\Omc^-_T$ consists
of all instantiations of the first three lines of
Figure~\ref{fig:second}
and that for $\ELU_\bot$, the same is true for
Figure~\ref{fig:third}. However, the first three lines of these
figures are identical. 
%
\begin{restatable}{lemma}{lemcompaux} \label{lem:compaux} $\Omc_T^-
  \models C_0 \sqsubseteq \mn{Dis}^\EL_{\Omc_S}(C_0)^\uparrow$ for every
  \EL concept~$C_0$ over $\mn{sig}(\Omc_S)$.
\end{restatable}

We prove Lemma~\ref{lem:compaux} by first introducing a special chase
procedure for $\ELU_\bot$ ontologies that is specifically tailored
towards our approximations. Unlike more standard chase procedures for
$\ELU_\bot$, our chase is deterministic rather than disjunctive.

We define an entailment notion $\Amc \vdash C(a)$ between extended
$\ELU_\bot$ ABoxes \Amc and $\ELU_\bot$ concepts $C$ as follows:
\begin{itemize}

\item $\Amc \vdash \top(a)$ always holds;

\item $\Amc \vdash \bot(a)$ if $\bot(b) \in \Amc$ for some $b$;

\item $\Amc \vdash A(a)$ if $A(a) \in \Amc$;

\item $\Amc \vdash C_1 \sqcap C_2(a)$ if $\Amc \vdash C_1(a)$ and
  $\Amc \vdash C_2(a)$;

\item $\Amc \vdash C_1 \sqcup C_2(a)$ if 
$C_1 \sqcup C_2(a) \in \Amc$;

\item $\Amc \vdash \exists r . C(a)$ if there is $b$ such that
  $r(a,b) \in \Amc$ and $\Amc \vdash C(b)$.

\end{itemize}
Note that if $C$ is an \EL concept, then $\Amc \vdash C(a)$ if $a \in
C^{\Imc_{\Amc^-}}$ where $\Imc_{\Amc^-}$ is $\Amc^-$ viewed as an
interpretation in the obvious way. Let $\mn{Dis}_\Omc(C)$ be defined
in the same way as $\mn{Dis}^\EL_\Omc(C)$ except that all concepts in
$\mn{sub}(\Omc)$ that are concept names or of the form $\exists r . E$
are considered instead of only \EL concepts of this form.


\smallskip 

Let $\Amc$ be an ABox and \Omc an $\ELU_\bot$ ontology. The chase
produces produces a sequence of ABoxes $\Amc = \Amc_0,
\Amc_1,\Amc_2,\cdots$ such that $\Amc_i \subseteq \Amc_{i+1}$ for all
$i \geq 0$. Although different sequences can be produced, the limit
$\bigcup_{i \geq 0} \Amc_i$ will be unique and we call it the
\emph{result} of chasing \Amc with \Omc. We call an individual in
$\bigcup_{i \geq 0} \Amc_i$ \emph{original} if it already occurs in
\Amc and \emph{anonymous} otherwise. In the ABoxes $\Amc_i$, anonymous
individuals can be marked or not. Each ABox $\Amc_{i+1}$ is obtained
from $\Amc_i$ by \emph{chasing a single step} with~$\Omc$, that is,
$\Amc_{i+1}$ is obtained from $\Amc_i$ in one of the following ways:
\begin{enumerate}

\item choose $C \sqsubseteq D \in \Omc$ and $a \in \mn{Ind}(\Amc)$
  with $\Amc \vdash C(a)$ and add $\mn{DNF}(D)(a)$; 

\item choose $C_1 \sqcap C_2(a) \in \Amc$
  and add $C_1(a),C_2(a)$;

\item choose $\exists r . C(a) \in \Amc$ 
  and add $r(a,b), C(b)$ for a fresh $b$; we say that $b$ was
  \emph{introduced for~$C$};





\item choose $D_1(a) \in \Amc$ with $D_1 \in \mn{Dis}^-(\Omc)$ and
  $D_2,D_3 \in \mn{Dis}(\Omc)$ such that $\Amc \vdash D_2(a)$
  and
  $\Omc \models D_1 \sqcap D_2 \sqsubseteq D_3$, and add $D_3(a)$;


\item choose $r(a,b),D_1(b) \in \Amc$ with $D_1 \in \mn{Dis}^-(\Omc)$
  and $D_2 \in \mn{Dis}(\Omc)$
    such that
  $\Omc \models \exists r . D_1 \sqsubseteq D_2$
  and add $D_2(a)$;

\item choose $D(a) \in \Amc$ with $D \in \mn{Dis}^-(\Omc)$ and $a$ anonymous and introduced for $C$, and add $\mn{Dis}_\Omc(C)(a)$; 
  mark $a$;

\item choose $r(a,b) \in \Amc$ with $a$ anonymous 
and introduced for
  $C_a$ 
and $b$ marked, anonymous, and introduced for $C_b$ such that
$\mn{Dis}_\Omc(\exists r .C_b) \in \mn{Dis}^-(\Omc_S)$; add
  $\mn{Dis}_\Omc(C_a)(a)$; 
 mark $a$ if it is anonymous.
%


\end{enumerate}
Note that Rules 1-3 implement Line 1 of our \EL approximations
$\Omc^\ell_T$ while Rules~4 and~5 correspond to Lines~2 and~3 of the
approximation. Rules 6-7 are there to deal with anonymous individuals
which behave in a different way than original ones.  

We require that the chase is \emph{fair}, that is, every possible way
to chase a single step is eventually used. Note that our chase is
oblivious, that is, a chase rule can be applied even if its
`consequence' is already there. This implies that the results of the
chase, which we denote with $\mn{ch}^{\text{sp}}_{\Omc}(\Amc)$, is
unique up to isomorphism.

The main property that we require of the chase is the following
completeness
property. 
%
%
 %
\begin{restatable}{lemma}{lemchasecompl}
\label{lem:chasecompl}
%
%
Let $\Omc$ be an $\ELU_\bot$ ontology and $C_0$ an \EL concept over
$\mn{sig}(\Omc_S)$. Then $\mn{ch}^{\text{sp}}_{\Omc}(\Amc_{C_0})
\vdash \mn{Dis}^\EL_\Omc(C_0)(a_0)$.
\end{restatable}
Since the proof of Lemma~\ref{lem:chasecompl} is rather laborious, we
defer it to Section~\ref{app:chasecompl}.

\medskip

We now return to the proof of Lemma~\ref{lem:compaux}. 
Let $\Amc_0,\Amc_1,\dots$ be a sequence of ABoxes generated by chasing
$\Amc_{C_0}=\Amc_0$ with~$\Omc_S$ using the special chase introduced
above. It is easy to see that all extended ABoxes
$\Amc_0,\Amc_1,\dots$ are ditree-shaped and can thus be viewed as
$\ELU_\bot$ concepts $C_0,C_1,\dots$ in which all disjunctions are
from $\mn{Dis}(\Omc_S)$. Note that also the ABox assertions $C(a) \in
\Amc_i$ with $C$ a compound concept or $\bot$ give raise to
subconcepts in~$C_i$.
\\[2mm]
{\bf Claim.} $\Omc_T^- \models C^\uparrow_i \sqsubseteq
C^\uparrow_{i+1}$, for all $i \geq 0$.
\\[2mm]
To prove the claim, let $i \geq 0$. We make a case distinction
according to the chase rule with which $\Amc_{i+1}$ is obtained
from~$\Amc_i$:
\begin{enumerate}

\item Then there is a $C \sqsubseteq D \in \Omc_S$ and an $a \in
  \mn{Ind}(\Amc)$ such that $\Amc_i \vdash C(a)$ and
  $\Amc_{i+1}=\Amc_i \cup \{ \mn{DNF}(D)(a) \}$. Let $E_a$ be the
  subconcept of $C_i$ that corresponds to the subtree rooted at $a$ in
  $\Amc_i$ and let $F_a$ be the subconcept of $C_{i+1}$ that
  corresponds to the subtree rooted at $a$ in $\Amc_{i+1}$. Since $C$
  is an \EL concept, $\Amc_i \vdash C(a)$ implies $\Amc^-_i \vdash
  C(a)$.  Consequently, $\emptyset \models E_a^\uparrow \sqsubseteq
  C$. Moreover, $F_a=E_a \sqcap \mn{DNF}(D)$ and $\Omc_T^-$ contains
  the CI $C \sqsubseteq \mn{DNF}(D)^\uparrow$, thus $\Omc_T^- \models
  C^\uparrow_i \sqsubseteq C^\uparrow_{i+1}$ as required.

\item Trivial.

\item Trivial.

\item Then there are $D_1(a) \in \Amc_i$ with $D_1 \in \mn{Dis}^-(\Omc_S)$
  and $D_2,D_3 \in \mn{Dis}(\Omc_S)$ such that $\Amc_i \vdash D_2(a)$,
  $\Omc_S \models D_1 \sqcap D_2 \sqsubseteq D_3$, and
  $\Amc_{i+1}=\Amc_i \cup \{ D_3(a) \}$. Let $E_a$ be the subconcept
  of $C_i$ that corresponds to the subtree rooted at $a$ in $\Amc_i$
  and let $F_a$ be the subconcept of $C_{i+1}$ that corresponds to the
  subtree rooted at $a$ in $\Amc_{i+1}$. Then $F_a = E_a \sqcap D_3$.
  From $D_1(a) \in \Amc$ and $D_1 \in \mn{Dis}^-(\Omc_S)$, we obtain
  that $X_{D_1}$ is a top-level conjunct of $E_a^\uparrow$. From $\Amc_i \vdash
  D_2(a)$, we obtain that $\emptyset \models E^\uparrow_a \sqsubseteq
  D_2^\uparrow$. Moreover, $\Omc_T^-$ contains the CI $X_{D_1} \sqcap
  D_2^\uparrow \sqsubseteq D_3^\uparrow$, and thus $\Omc_T^- \models
  C^\uparrow_i \sqsubseteq C^\uparrow_{i+1}$ as required.

\item Similar to the previous case, using the third line of $\Omc_T^-$.

\item Then there is a $D(a) \in \Amc_i$ with $D \in \mn{Dis}^-(\Omc_S)$,
  $a$ anonymous and introduced for $C$, and $\Amc_{i+1}=\Amc_i \cup \{
  \mn{Dis}_{\Omc_S}(C)(a) \}$. Let $E_a$ be the subconcept of $C_i$
  that corresponds to the subtree rooted at $a$ in $\Amc_i$ and let
  $F_a$ be the subconcept of $C_{i+1}$ that corresponds to the subtree
  rooted at $a$ in $\Amc_{i+1}$. Since $a$ was introduced for $C$,
  $C(a) \in \Amc_i$ and thus $\emptyset \models E_a \sqsubseteq C$. Since $C \in
  \mn{Dis}(\Omc_S)$ and $\Omc_S \models C \sqsubseteq
  \mn{Dis}_{\Omc_S}(C)$, $\Omc^-_T$ contains the CI $X_D \sqcap
  C^\uparrow \sqsubseteq \mn{Dis}_{\Omc_S}(C)$. Moreover, $F_a = E_a
  \sqcap \mn{Dis}_{\Omc_S}(C)$.  It follows that
  $\Omc_T^- \models C^\uparrow_i \sqsubseteq C^\uparrow_{i+1}$.

\item Similar to the previous case.

\end{enumerate}
This finishes the proof of the claim.

\medskip

By Lemma~\ref{lem:chasecompl},
$\mn{ch}^{\text{sp}}_{\Omc_S}(\Amc_{C_0}) \vdash
\mn{Dis}^\EL_{\Omc_S}(C_0)(a_0)$ and thus $\Amc_{i} \vdash
\mn{Dis}^\EL_{\Omc_S}(C_0)(a_0)$ for some $i$. First assume that
$\mn{Dis}^\EL_{\Omc_S}(C_0)$ contains more than one disjunct. Then, by
definition of $\vdash$, $\mn{Dis}^\EL_{\Omc_S}(C_0)(a_0) \in \Amc_i$
and thus $X_{\mn{Dis}^\EL_{\Omc_S}(C_0)}=
\mn{Dis}^\EL_{\Omc_S}(C_0)^\uparrow$ is a top-level conjunct of
$C^\uparrow_i$ implying $\emptyset \models C^\uparrow_i \sqsubseteq
X^\EL_{\mn{Dis}_{\Omc_S}(C_0)}$. From the claim and
$C_0=C_0^\uparrow$, we obtain $\Omc_T^- \models C_0 \sqsubseteq
C^\uparrow_i$ and are done. Now assume that
$\mn{Dis}^\EL_{\Omc_S}(C_0)$ contains a single disjunct. Then
$\Amc_{i} \vdash K$ for each conjunct $K$ of
$\mn{Dis}^\EL_{\Omc_S}(C_0)$. By definition of `$\vdash$' and
$C^\uparrow_i$, it follows that
$\emptyset \models C^\uparrow_i \sqsubseteq K^\uparrow$ for each such $K$, and
thus $\emptyset \models C_i^\uparrow \sqsubseteq
\mn{Dis}^\EL_{\Omc_S}(C_0)^\uparrow$. It again remains to apply the
claim. Finally assume that $\mn{Dis}^\EL_{\Omc_S}(C_0) =\bot$.  Then
$\bot(b) \in \Amc_i$ for some $b$ and thus $\emptyset \models
C_i^\uparrow \sqsubseteq \bot$. We can once more apply the claim.
This finishes the proof of Lemma~\ref{lem:compaux}. 

\medskip 

Now back to the proof of completeness, that is, of
Lemma~\ref{lem:comp}. We need some more preliminaries.
\begin{lemma}
  \label{lem:generatablesem}
  Let \Omc be an \ELU ontology and $C, \exists r . D$ \EL concepts.
  If $\Omc \models C \sqsubseteq \exists r . D$ and $C$ contains no
  top-level conjunct $\exists r . C'$ such that $\Omc \models C'
  \sqsubseteq D$, then $D$ is \Omc-generatable.
\end{lemma}
\noindent
\begin{proof} \
Assume $\Omc \models C \sqsubseteq \exists r . D$ and $C$ contains no
top-level conjunct $\exists r . C'$ such that $\Omc \models C'
\sqsubseteq D$. Assume $D$ is not $\Omc$-generatable. Let 
  $$C= A_1 \sqcap \cdots \sqcap A_n \sqcap \exists r_1 . E_1
  \sqcap \cdots \sqcap \exists r_m . E_m.$$
  In $\Amc_C$, the root $a_0$ has outgoing edges $r_1(a_0,b_1),\dots,
  r_m(a_0,b_m)$. Extend $\Amc_{C}$ to a model $\Imc$ as follows:
\begin{enumerate}
\item add for any $b_{i}$, $1\leq i \leq m$, a ditree-shaped model $\Imc_{b_{i}}$ of $\Omc$ with root $b_{i}$
such that $b_{i}\in E_{i}^{\Imc_{b_{i}}}$ and $b_{i}\not\in D^{\Imc_{b_{i}}}$;
\item add for any $\ELU$ concept $\exists r.E$ such that there is a CI $C' \sqsubseteq
D$ in $\Omc$ such that $D$ contains $\exists r . E$ as a top-level conjunct
an $r$-successor $a_{r,E}$ of $a_{0}$ and a ditree-shaped model $\Imc_{r,E}$ of $\Omc$ with root $a_{r,E}$
such that $a_{r,E}\in E^{\Imc_{r,E}}$ and $a_{r,E}\not\in D^{\Imc_{b_{i}}}$;
\item $a_{0}$ to $A^{\Imc}$ for any concept name $A$.
\end{enumerate}
Note that the interpretations $\Imc_{b_{i}}$ exist since $C$ contains no
top-level conjunct $\exists r . C'$ such that $\Omc \models C'
\sqsubseteq D$ and the interpretations $\Imc_{r,E}$ exist since we assume that $D$ is not $\Omc$-generatable.
By construction, $a_{0}\not\in (\exists r.D)^{\Imc}$ and $\Imc$ is a model of $\Omc$ as all nodes
distinct from $a_{0}$ clearly satisfy all CIs in $\Omc$ and $a_{0}$ satisfies all CIs in $\Omc$ by construction.
We have derived a contradiction to $\Omc \models C \sqsubseteq \exists r . D$ as $a_{0}\in C^{\Imc}$.
\end{proof}
For a ditree-shaped ABox \Amc and $k \geq 0$, we use $\Amc|_k$ to
denote the result of removing from \Amc all individuals on levels
larger than $k$ and $C_\Amc^a$ to denote the subABox of \Amc rooted at
$a$ viewed as an \EL concept. To prepare for the case of
$\ELU_\bot$-to-\EL approximations, we establish the following lemma
directly for $\ELU_\bot$ instead of for $\ELU$.
\begin{lemma}
\label{lem:smallC}
Let $\Omc$ be an $\ELU_\bot$ ontology such that all concepts
on the left hand side of CIs in $\Omc$ are \EL concepts. Let \Amc be a ditree-shaped ABox with root
$a_0$ such that $\Omc,\Amc \models \exists r . C(a_0)$, $C$~an
\EL concept of depth $k$. Let $\Amc^\pm$ be the extended ABox obtained
from $\Amc|_k$ by adding $\mn{Dis}^\EL_\Omc(C^a_\Amc)$ whenever $a$ is a
leaf in $\Amc|_k$. Then $\Omc,\Amc^\pm \models \exists r . C(a_0)$.
\end{lemma}
\noindent
\begin{proof} \ 
  Assume that
  $\Omc,\Amc^\pm \not\models \exists r . C(a_0)$. Take a ditree shaped
  model $\Imc$ of $\Omc$ and $\Amc^\pm$ with
  $a_{0}\not\in (\exists r.C)^{\Imc}$.  Let $a$ be a node of depth $k$
  in $\Amc$. We have $a \in \mn{Dis}^{\EL}_\Omc(C^a_\Amc)^\Imc$ and thus
  there is a disjunct $D$ of $\mn{Dis}^{\EL}_\Omc(C^a_\Amc)$ with
  $a \in D^\Imc$. Let $\Amc_{a}$ be the subABox of $\Amc$ rooted at $a$.
  Observe that $\Amc_{a}$ is satisfiable w.r.t.~$\Omc$: otherwise $\bot$ is the only
  disjunct of $\mn{Dis}^{\EL}_{\Omc}(C^a_\Amc)$ and so $\Amc^{\pm}$ is not satisfiable.
  Thus $\Omc,\Amc^\pm \models \exists r . C(a_0)$, and we have derived a contraction.
  As $\Amc_{a}$ is satisfiable w.r.t.~$\Omc$ we obtain by definition of $\mn{Dis}^{\EL}_{\Omc}(C^a_\Amc)$ 
  that there is a model $\Jmc_a$ of $\Omc$ and $\Amc_{a}$ 
  such that whenever $a \in E^{\Jmc_a}$ for some
  $\EL$ concept $E \in \mn{sub}^-(\Omc)$, then $a \in E^\Imc$. Construct a new
  interpretation $\Imc'$ by adding to $\Imc$ the interpretation
  $\Jmc_{a}$, for all nodes $a$ of depth $k$ in $\Amc$ (where $\Imc$
  and $\Jmc_{a}$ only share $a$). $\Imc'$
  is a model of $\Omc$ and $\Amc$ since $a\in E^{\Imc}$ if $a\in E^{\Imc'}$, for all
  \EL concepts $E\in \mn{sub}^-(\Omc)$ and $a$ of depth $k$ in $\Amc$. Moreover,
  $a_{0}\not\in (\exists r.C)^{\Imc'}$, as required.
\end{proof}

\medskip

We are now in a position to prove Lemma~\ref{lem:comp}.

\lemcomp*

\noindent
\begin{proof}\ Assume that $\Omc_S \models C_0 \sqsubseteq D_0$ with
  $C_0,D_0$ \EL concepts over $\mn{sig}(\Omc_S)$ such that the role
  depth of $D_0$ is bounded by~$\ell$. It clearly suffices to consider
  the cases where $D_0$ is a concept name and where it is of the form
  $\exists r . E_0$. 

  We start with the former, so let $D_0=A$.  Clearly,
  $\Omc_S \models C_0 \sqsubseteq A$ implies
  $\Omc_S \models \mn{Dis}^\EL_{\Omc_S}(C_0) \sqsubseteq A$.  It thus
  follows from Lemma~\ref{lem:compaux} that
  $O_T^\ell \models C_0 \sqsubseteq A$. To see this, first assume that
  $\mn{Dis}^\EL_{\Omc_S}(C_0)$ contains a single disjunct. Then $A$
  must be a conjunct of $\mn{Dis}^\EL_{\Omc_S}(C_0)$ and it suffices
  to apply Lemma~\ref{lem:compaux}. Now assume that
  $\mn{Dis}^\EL_{\Omc_S}(C_0)$ has more than one disjunct, that is, it
  is in $\mn{Dis}^-(\Omc_S)$. Then $\Omc^\ell_T$ contains the CI
  $X_{\mn{Dis}^\EL_{\Omc_S}(C_0)} \sqcap X_{\mn{Dis}^\EL_{\Omc_S}(C_0)}
  \sqsubseteq A$ (second line of Figure~\ref{fig:second}) and thus it
  again suffices to apply Lemma~\ref{lem:compaux}.

  The case where $D_0=\exists r . E_0$ is a consequence of the
  following claim. For each $a \in \mn{Ind}(\Amc_{C_0})$, we write
  $C^a_0$ as an abbreviation for $C^a_{\Amc_{C_0}}$.
  \\[2mm]
  {\bf Claim.}  For all $a \in \mn{Ind}(\Amc_{C_0})$ and \EL concepts
  $\exists r . E$ of depth $\ell-\mn{depth}(a)$, $\Omc_S \models
  C^a_0 \sqsubseteq \exists r . E$ implies $\Omc^\ell_T
  \models C^a_0 \sqsubseteq \exists r . E$.
  \\[2mm]
  \emph{Proof of claim}. The proof is by induction on the co-depth of~$a$. 

  \smallskip 
  \noindent 
  \emph{Induction start}. Then $a$ is a leaf in $\Amc_{C_0}$ and thus
  $C^a_0$ does not have any top-level conjuncts of the form
  $\exists r . E'$.  Lemma~\ref{lem:generatablesem} thus yields that
  $E$ is $\Omc_\Smc$-generatable.  Thus $C^a_0 \sqsubseteq
  \exists r . E$ is a CI in $\Omc_T^\ell$.

  \smallskip 
  \noindent 
  \emph{Induction step.} Then $a$ is a non-leaf in $\Amc_{C_0}$. We
  distinguish two cases.

  \medskip
  \noindent
  \emph{Case~1}. There is a top-level conjunct $\exists r . E'$ in
  $C^a_0$ such that $\Omc_\Smc \models E' \sqsubseteq
  E$. Then $a$ has an $r$-successor $b$ in $\Amc_{C_0}$ such that
  $C^b_0 = E'$.  Let
  $$
  E=A_1 \sqcap \cdots \sqcap A_{n} \sqcap \exists r_1 . E_1 \sqcap
  \cdots \sqcap \exists r_{m} . E_{m}.
  $$
  Since we have already shown Lemma~\ref{lem:comp} for the case where
  $D_0$ is a concept name, we obtain $\Omc_T^\ell \models
  C^b_0 \sqsubseteq A_i$ for $1 \leq i \leq n$.  From the
  induction hypothesis, we further obtain $\Omc_T^\ell \models
  C^b_0 \sqsubseteq \exists r_i . E_i$ for $1 \leq i \leq
  m$. Thus $\Omc_T^\ell \models C^b_0 \sqsubseteq E$ and
  consequently $\Omc_T^\ell \models C^a_0 \sqsubseteq
  \exists r . E$ as required.
  
  \medskip
  \noindent
  \emph{Case~2}. There is no top-level conjunct $\exists r . E'$ in
  $C^a_0$ such that $\Omc_\Smc \models E' \sqsubseteq E$.
  Then Lemma~\ref{lem:generatablesem} yields that $E$ is
  $\Omc_\Smc$-generatable. Let $\Amc$ be the ditree-shaped subABox of
  $\Amc_{C_0}$ rooted at $a$ and let $\Amc^\pm$ be the extended ABox
  obtained from $\Amc|_{k}$, with $k$ the depth of $\exists r . E$, by
  adding $\mn{Dis}^\EL_{\Omc_S}(C^c_\Amc)^\uparrow(c)$ whenever $c$ is a
  leaf in $\Amc|_{k}$. Applying Lemma~\ref{lem:smallC} to \Amc and
  $\Amc^\pm$ and with $\exists r .E$ in place of $\exists r . C$, we
  obtain $\Omc_S,\Amc^\pm \models \exists r . E(a)$.  Let $C^\pm$ be
  $\Amc^\pm$ viewed as an \ELU concept. Then $\Omc_S \models C^\pm
  \sqsubseteq \exists r . E$. Since $E$ is $\Omc_S$-generatable,
  $(C^\pm)^\uparrow \sqsubseteq \exists r . E$ is thus a CI in
  $\Omc_T^\ell$. We next observe that, by Lemma~\ref{lem:compaux},
  $\Omc_T^\ell\models C^c_\Amc \sqsubseteq
  \mn{Dis}^\EL_{\Omc_S}(C^c_\Amc)^\uparrow$ and thus $\Omc_T^\ell,
  \Amc\models \mn{Dis}^\EL_{\Omc_S}(C^c_\Amc)^\uparrow(c)$ for all leaves
  $c$ in $\Amc|_{k}$.  Together with the construction of $\Amc^\pm$
  and $C^\pm$, this yields that $\Omc^\ell_T \models C^a_0
  \sqsubseteq (C^\pm)^\uparrow$. Together with $(C^\pm)^\uparrow
  \sqsubseteq \exists r . E$ being a CI in $\Omc_T^\ell$, we obtain
  $\Omc^\ell_T \models C^a_0 \sqsubseteq \exists r. E$ as
  required.
\end{proof}

\subsection{Soundness and Completeness of the Special Chase}
\label{app:chasecompl}

Our main aim is to establish Lemma~\ref{lem:chasecompl}.  We start,
however, with proving soundness of the chase. While this is
interesting in its own right, we are not going to use it directly in
the context of approximations. It is, however, an ingredient to the
subsequent completeness proof.
%
%
%
%
%
\begin{lemma}
  \label{lem:chsound}
  Let $C_0$ be an \EL concept and \Omc an $\ELU_\bot$ ontology. Then
  $\mn{ch}^{\text{sp}}_{\Omc}(\Amc_{C_0}) \vdash D(a_0)$ implies
  $\Omc \models C_0 \sqsubseteq D$ for all 
  $D \in \mn{Dis}(\Omc)$.
%
%
\end{lemma}

\noindent
\begin{proof}\ Let $\Amc_{C_0}=\Amc_0,\Amc_1,\dots$ be a sequence
  generated by chasing $\Amc_{C_0}$ with \Omc using the special
  chase. Further, let \Imc be a model of \Omc and let
  $d \in C_0^\Imc$.  An \emph{extended homomorphism} from $\Amc_i$ to
  \Imc is a function $h:\mn{Ind}(\Amc) \rightarrow \Delta^\Imc$ such
  that
  \begin{enumerate}

  \item $C(a) \in \Amc_i$, $C$ potentially compound, implies
    $h(a) \in C^\Imc$ and

  \item $r(a,b) \in \Amc_i$ implies $(h(a),h(b)) \in r^\Imc$.

  \end{enumerate}
  We next observe the following.
  \\[2mm]
  {\bf Claim.} if $\Amc_i \vdash D(a)$, $a \in \mn{Ind}(\Amc_i)$ and
  $D \in \mn{Dis}(\Omc)$, and $h$ is an extended homomorphism from
  $\Amc_i$ to \Imc, then $h(a) \in D^\Imc$.
  \\[2mm]
  The claim can be proved by induction on the structure of $D$. If $D$
  takes the form $D_1 \sqcap D_2$ or $\exists r . D_1$, then this is
  straightforward using the semantics and induction hypothesis. If $D$
  is $\top$, $\bot$, a concept name, or of the form $D_1 \sqcup D_2$
  (note that in the latter case $\Amc_i \vdash D(a)$ implies
  $D(a) \in \Amc_i$), then this is immediate by definition of extended
  homomorphisms.

  \medskip

  We show by induction on $i$ that for each $i \geq 0$, there is an
  extended homomorphism $h_i$ from $\Amc_i$ to \Imc with
  $h(a_0)=d$. This is trivial for $i=0$ since $d \in C_0^\Imc$. For
  $i \geq 0$, we make a case distinction according to the rule that
  was applied to obtain $\Amc_{i+1}$ from $\Amc_i$:
  \begin{enumerate}

  \item Then there is a $C \sqsubseteq D \in \Omc$ and an $a \in
    \mn{Ind}(\Amc)$ such that $\Amc_i \vdash C(a)$ and
    $\Amc_{i+1}=\Amc_i \cup \{ \mn{DNF}(D)(a) \}$.  By the claim,
    $\Amc_i \vdash C(a)$ implies $h_i(a) \in C^\Imc$.  Since \Imc is a
    model of \Omc, $h_i(a) \in D^\Imc = \mn{DNF}(D)^\Imc$ and
    consequently $h_i$ can be extended to an extended homomorphism
    $h_{i+1}$ from $\Amc_{i+1}$ to \Imc.

\item Trivial.

\item Trivial.

\item Then there are $D_1(a) \in \Amc_i$ with $D_1 \in \mn{Dis}^-(\Omc)$
  and $D_2,D_3 \in \mn{Dis}(\Omc)$ such that $\Amc_i \vdash D_2(a)$,
  $\Omc \models D_1 \sqcap D_2 \sqsubseteq D_3$, and
  $\Amc_{i+1}=\Amc_i \cup \{ D_3(a) \}$. 
  From $D_1(a) \in \Amc_i$, $\Amc_i \vdash D_2(a)$, and the claim,
  we get $h_i(a) \in (D_1 \sqcap D_2)^\Imc$. Since \Imc is a model of
  \Omc, $h_i(a) \in D_3(a)$. Thus, we can choose $h_{i+1}=h_i$.

\item Similar to the previous case.

\item Then there is a $D(a) \in \Amc_i$ with $D \in \mn{Dis}^-(\Omc)$,
  $a$ anonymous and introduced for $C$, and $\Amc_{i+1}=\Amc_i \cup \{
  \mn{Dis}_{\Omc}(C)(a) \}$. We have $C(a) \in \Amc_i$, and thus $h(a)
  \in C^\Imc$. Since \Imc is a model of \Omc, this implies $h(a) \in
  \mn{Dis}_{\Omc}(C)^\Imc$ and thus we can choose $h_{i+1}=h_i$.

\item Similar to the previous case. 

  \end{enumerate}
  We now finish the proof of Lemma~\ref{lem:chsound}. Let
  $\mn{ch}^{\text{sp}}_{\Omc}(\Amc_{C_0}) \vdash D(a_0)$ with $D \in
  \mn{Dis}(\Omc)$. Then there is an $\Amc_i$ with $\Amc_i \vdash
  D(a_0)$. From the claim, we obtain $d=h_i(a_0) \in C^\Imc$. Since
  this holds for all \Imc and $d$, we have shown that $\Omc \models
  C_0 \sqsubseteq D$, as required.
\end{proof}

\lemchasecompl*

\noindent
\begin{proof}\
  We start with a special case, which is that $\bot(b) \in
  \mn{ch}^{\text{sp}}_{\Omc}(\Amc_{C_0})$ for some $b$. By
  Lemma~\ref{lem:chsound}, $\Omc \models C_0 \sqsubseteq \bot$
  and thus $\mn{Dis}^\EL_\Omc(C_0) = \bot$. By definition of
  `$\vdash$', $\mn{ch}^{\text{sp}}_{\Omc}(\Amc_{C_0}) \vdash
  \bot(a_0)$ and thus we are done. In what follows, we can thus
  assume that $\bot(b) \notin
  \mn{ch}^{\text{sp}}_{\Omc}(\Amc_{C_0})$ for all $b$.
    
  To deal with the general case, assume to the contrary of what we
  have to prove that
  $\mn{ch}^{\text{sp}}_{\Omc}(\Amc_{C_0}) \not\vdash
  \mn{Dis}^\EL_\Omc(C_0)(a_0)$.  We are going to construct from
  $\mn{ch}^{\text{sp}}_{\Omc}(\Amc_{C_0})$ a model \Imc of \Omc with
  an element $d$ such that
  $d \in C_0^\Imc \setminus (\mn{Dis}^\EL_\Omc(C_0))^\Imc$, in
  contradiction to the definition of $\mn{Dis}^\EL_\Omc(C_0)$.
  %

\medskip 
  
An original $a \in \mn{Ind}(\mn{ch}^{\text{sp}}_{\Omc}(\Amc_{C_0}))$ is
\emph{disjunctive} if $\mn{ch}^{\text{sp}}_{\Omc}(\Amc_{C_0})$
contains at least one assertion $D(a)$ with $D \in \mn{Dis}^-(\Omc)$.
With each original disjunctive $a$, we associate a disjunction
$$
D_a = \mn{Dis}_\Omc(\bigsqcap \{ D \in \mn{Dis}(\Omc) \mid
\mn{ch}^{\text{sp}}_{\Omc}(\Amc_{C_0}) \vdash D(a)\}).
$$
For the above definition, it is important to note that
$\mn{sub}^-(\Omc) \subseteq \mn{Dis}(\Omc)$ and thus also all
$C \in \mn{sub}^-(\Omc)$ with
$\mn{ch}^{\text{sp}}_{\Omc}(\Amc_{C_0}) \vdash C(a)$
contribute to the definition of $D_a$. We observe the following:
\begin{itemize}

\item[(P1)] If $\mn{ch}^{\text{sp}}_{\Omc}(\Amc_{C_0}) \vdash D(a)$
  with $D \in \mn{Dis}(\Omc)$, then $\emptyset \models D_a \sqsubseteq D$.

  $\mn{ch}^{\text{sp}}_{\Omc}(\Amc_{C_0}) \vdash D(a)$ implies
  %
  $$\emptyset \models
\bigsqcap \{ D' \in \mn{Dis}(\Omc) \mid
\mn{ch}^{\text{sp}}_{\Omc}(\Amc_{C_0}) \vdash D'(a_0)\} \sqsubseteq D.
  $$
  By definition of $\mn{Dis}_\Omc$,
  $$
  \emptyset \models D_a \sqsubseteq \bigsqcap \{ D' \in \mn{Dis}(\Omc) \mid
\mn{ch}^{\text{sp}}_{\Omc}(\Amc_{C_0}) \vdash D'(a_0)\}  
  $$
  and thus 
%
  $\emptyset \models D_a \sqsubseteq D$. 
  
\item[(P2)] If $\emptyset \models D_a \sqsubseteq  D \in \mn{Dis}(\Omc)$, then
  $D(a) \in \mn{ch}^{\text{sp}}_{\Omc}(\Amc_{C_0})$.

  Since $a$ is disjunctive, there is some
  $D_1(a) \in \mn{ch}^{\text{sp}}_{\Omc}(\Amc_{C_0})$ with
  $D_1 \in \mn{Dis}^-(\Omc)$. Let $D_1,\dots,D_k$ be all disjunctions
  from $\mn{Dis}(\Omc)$ with
  $\mn{ch}^{\text{sp}}_{\Omc}(\Amc_{C_0}) \vdash D_i(a)$. Consider
  $\mn{Dis}_\Omc(D_1 \sqcap D_2)$. In the very special case that this
  disjunction consists of a single disjunct that contains all concepts
  from $\mn{sub}^-(\Omc)$ as conjuncts,
  $\mn{Dis}_\Omc(D_1 \sqcap D_2)=D_a$ and Rule~4 applied to $D_1$ and
  $D_2$ yields $D_a(a) \in \mn{ch}^{\text{sp}}_{\Omc}(\Amc_{C_0})$ as
  required. Otherwise, we find some $D'$ with at least two disjuncts
  such that
  $\emptyset \models \mn{Dis}_\Omc(D_1 \sqcap D_2) \sqsubseteq
  D'_2$. We can apply Rule~4 again to show
  $D'_2(a) \in \mn{ch}^{\text{sp}}_{\Omc}(\Amc_{C_0})$.  Since $D'_2$
  has at least two disjuncts, we can proceed in the same way applying
  Rule~4 to $D'_2,D_3$, then to $D'_3,D_4$, and so on. In the last
  step, we can clearly choose $D_a$ as $D'_k$.
  Finally, another application of Rule~4 with $D_1=D_2=D_a$ and
  $D_3=D$ yields $D(a) \in \mn{ch}^{\text{sp}}_{\Omc}(\Amc_{C_0})$.

\end{itemize}
Note that it follows from (P2) and the assumption that $\bot(a)
\notin \mn{ch}^{\text{sp}}_{\Omc}(\Amc_{C_0})$ that $D_a \neq \bot$,
that is, $D_a$ has at least one disjunct.

\smallskip

We now consider each original disjunctive
$a \in \mn{Ind}(\mn{ch}^{\text{sp}}_{\Omc}(\Amc_{C_0}))$, identify a
disjunct $E_a$ of $D_a$ and extend
$\mn{ch}^{\text{sp}}_{\Omc}(\Amc_{C_0})$ with $D(a)$ for each
$D \in \mn{Dis}(\Omc)$ with $\emptyset \models E_a \sqsubseteq D$.
We then show that no new applications of
chase rules are possible afterwards, with the possible exception of
applications of Rule~3 to original disjunctive individuals $a$. We
also select an $E_a$ for each original non-disjunctive $a$, in a
trivial way: $E_a$ is then the conjunction of all
$C \in \mn{sub}^-(\Omc)$ such that
$\mn{ch}^{\text{sp}}_{\Omc}(\Amc_{C_0}) \vdash C(a)$.

\medskip 

We start at the root $a_0$ (if it is disjunctive).  Recall our
assumption that
$\mn{ch}^{\text{sp}}_{\Omc}(\Amc_{C_0}) \not\vdash
\mn{Dis}^\EL_\Omc(C_0)(a_0)$. There must be a disjunct $E_{a_0}$ of
$D_{a_0}$ such that
$\emptyset \not\models E_{a_0} \sqsubseteq \mn{Dis}^\EL_\Omc(C_0)$ as
otherwise
$\emptyset \models D_{a_0} \sqsubseteq \mn{Dis}^\EL_\Omc(C_0)$ and
thus (P2) yields
$\mn{Dis}^\EL_\Omc(C_0)(a_0) \in
\mn{ch}^{\text{sp}}_{\Omc}(\Amc_{C_0})$. Together with Rules~2 and~3,
this yields
$\mn{ch}^{\text{sp}}_{\Omc}(\Amc_{C_0}) \vdash
\mn{Dis}^\EL_\Omc(C_0)(a_0)$, a contradiction.
%
%
Let $\Amc^+$ denote the result of extending
$\mn{ch}^{\text{sp}}_{\Omc}(\Amc_{C_0})$ for $E_{a_0}$ as described
above. We observe the following counterpart of (P1) for $\Amc^+$.
\begin{itemize}

\item[(P3)] if $\Amc^+ \vdash D(a_0)$ with $D \in \mn{Dis}(\Omc)$,
  then $\emptyset \models E_{a_0}
  \sqsubseteq D$.

  If $\mn{ch}^{\text{sp}}_{\Omc}(\Amc_{C_0}) \vdash D(a_0)$, then this
  follows from (P1).  Otherwise, by definition of $\vdash$ and
  construction of $\Amc^+$,
  we must have $\emptyset \models \bigsqcap S \sqsubseteq D$
  where $S$ contains
  \begin{enumerate}

  \item all concepts $D' \in \mn{Dis}(\Omc)$ such that
    $\mn{ch}^{\text{sp}}_{\Omc}(\Amc_{C_0}) \vdash D'(a_0)$ and

  \item all concepts $D'$ with $D'(a_0)$ fresh in $\Amc^+$.

  \end{enumerate}
  (P1) implies   $\emptyset \models E_{a_0} \sqsubseteq D'$ for all concepts
  $D'$ from Point~1 and the construction of $\Amc^+$ yields
  $\emptyset \models E_{a_0} \sqsubseteq D'$ for all concepts
  $D'$ from Point~2. Thus $\emptyset \models E_{a_0} \sqsubseteq D$.

\end{itemize}
We show that no new rule applications are possible, except for 
applications of Rule~3 to original disjunctive individuals:
\begin{itemize}

\item Rule~1. Assume that $\Amc^+ \vdash C(a_0)$ and that $C
  \sqsubseteq D \in \Omc$. Then (P3) yields $\emptyset \models E_{a_0}
  \sqsubseteq C$.  By definition of $D_{a_0}$, this implies that $C$
  is a conjunct of $E_{a_0}$.  Let $D$ be of the form $C_1 \sqcap
  \cdots \sqcap C_n \sqcap D_1 \sqcap \cdots \sqcap D_m$ where
  $C_1,\dots,C_n$ are concept names or existential restrictions and
  $D_1,\dots,D_m$ are disjunctions. Then $C_1,\dots,C_n$ must also be
  conjuncts in $E_{a_0}$. Moreover, for $1 \leq i \leq m$,
  $\mn{DNF}(D_i)$ must contain a disjunct $G$ such that all conjuncts
  of $G$ are in $E_{a_0}$. This implies $\emptyset \models E_{a_0}
  \sqsubseteq \mn{DNF}(D)$.  Consequently, $\mn{DNF}(D)(a_0)$ is in $\Amc^+$.

\item Rule~2. If $C_1 \sqcap C_2(a_0)$ is fresh in $\Amc^+$, then $C_1
  \sqcap C_2 \in \mn{Dis}(\Omc)$ is such that $\emptyset \models E_{a_0} \sqsubseteq C_1
  \sqcap C_2$. Thus $\emptyset \models E_{a_0} \sqsubseteq C_i$ for $i \in \{1,2\}$ and as a
  consequence, $C_1(a_0),C_2(a_0)$ are also in~$\Amc^+$.

\item Rule~3. New applications of Rule~3 are possible only to $a_0$,
  which is original and disjunctive.
  
\item Rule~4. Assume that $D_1(a_0)$ is in $\Amc^+$, $D_1 \in
  \mn{Dis}^-(\Omc)$, that
  $\Amc^+ \vdash D_2(a_0)$, and that
  $\Omc \models D_1 \sqcap D_2 \sqsubseteq D_3$. We have
  $\emptyset \models E_{a_0} \sqsubseteq D_1$: if $D_1(a_0)$ is fresh in $\Amc^+$, then this is
  clear; otherwise, $D_1(a_0) \in
  \mn{ch}^{\text{sp}}_{\Omc}(\Amc_{C_0})$, which implies
  $\mn{ch}^{\text{sp}}_{\Omc}(\Amc_{C_0}) \vdash D_1(a_0)$ since
 $D_1 \in
  \mn{Dis}^-(\Omc)$, and thus
  (P1) yields $\emptyset \models E_{a_0} \sqsubseteq D_1$. Moreover, $\emptyset \models E_{a_0} \sqsubseteq
  D_2$ by (P3).
  By definition of $D_{a_0}$, $\emptyset \models E_{a_0} \sqsubseteq D_3$ and thus
  $D_3(a_0)$ is in $\Amc^+$.

\item Rule~5. Trivially not applicable since $a_0$ has no predecessors. 

\item Rule~6 and~7. Only apply to anonymous individuals, but $a_0$ is
  original.

\end{itemize}
This finishes the extension of
$\mn{ch}^{\text{sp}}_{\Omc}(\Amc_{C_0})$
at $a_0$. From now on, we assume that this extension has been
incorporated into $\mn{ch}^{\text{sp}}_{\Omc}(\Amc_{C_0})$, that
is, we write $\mn{ch}^{\text{sp}}_{\Omc}(\Amc_{C_0})$ in place of
$\Amc^+$. Trivially, property (P1) still holds for all $a \neq a_0$
and property (P2) is preserved.

\medskip

We then apply the following extension as long as possible. Choose some $r(b,a) \in \mn{ch}^{\text{sp}}_{\Omc}(\Amc_{C_0})$ with
$a$ original and disjunctive and assume that $E_b$ was already
determined and $\mn{ch}^{\text{sp}}_{\Omc}(\Amc_{C_0})$ extended
accordingly (the latter only if $b$ is disjunctive). We argue that there must be a
disjunct $E_a$ of $D_a$ such that the following properties are
satisfied:
\begin{enumerate}


\item[(a)] $\emptyset \models E_a \sqsubseteq C$ and $\exists r . C \in \mn{sub}(\Omc)$ implies
$ \mn{ch}^{\text{sp}}_{\Omc}(\Amc_{C_0}) \vdash \exists r . C(b)$;

\item[(b)] $\emptyset \models E_a \sqsubseteq D_1 \in \mn{Dis}^-(\Omc)$ and $\Omc \models \exists r
  . D_1 \sqsubseteq D_2$ with $D_2 \in \mn{Dis}(\Omc)$ implies that
  $D_2(b) \in \mn{ch}^{\text{sp}}_{\Omc}(\Amc_{C_0})$.

\end{enumerate}
Assume that this is not the case. Let $D_a = E_1 \sqcup \cdots \sqcup
E_k$. For $1 \leq i \leq k$, we then find one of the following:
\begin{itemize}

\item[(i)] $D'_i= \exists r . D_i \in \mn{sub}(\Omc)$ with $\emptyset \models E_i \sqsubseteq D_i$ and
  $\mn{ch}^{\text{sp}}_{\Omc}(\Amc_{C_0}) \not\vdash D'_i(b)$.

\item[(ii)] $D_i \in \mn{Dis}^-(\Omc)$ and $D'_i \in \mn{Dis}(\Omc)$
  such that $\emptyset \models E_i \sqsubseteq D_i$, $\Omc \models
  \exists r . D_i \sqsubseteq D'_i$, and
  $D'_i(b) \notin\mn{ch}^{\text{sp}}_{\Omc}(\Amc_{C_0})$.

\end{itemize}
Let $D_L$ denote the result of removing identical disjuncts from
$\bigsqcup_{1 \leq i \leq k} D_i$ and $D_R$ the result of removing
identical disjuncts from $\bigsqcup_{1 \leq i \leq k} D'_i$. We have
$D_L,D_R \in \mn{Dis}(\Omc)$ while this need not be true for the
disjunctions that they have been obtained from. Clearly,
$ \Omc \models \exists r . D_L \sqsubseteq D_R$. Since each $D_i$ is
from $\mn{Dis}^-(\Omc)$, $D_L \in \mn{Dis}^-(\Omc)$ while this is not
guaranteed for $D_R$ even when $k>1$.  Since
$\Omc \models \exists r . D_L \sqsubseteq D_R$ and since $D_L$ is from
$\mn{Dis}^-(\Omc)$, Rule~5 yields
$D_R(b) \in \mn{ch}^{\text{sp}}_{\Omc}(\Amc_{C_0})$.
due to Rules~2 and~3, this implies
$\mn{ch}^{\text{sp}}_{\Omc}(\Amc_{C_0}) \vdash D_R(b)$. 
We distinguish two cases.

First assume that $b$ is disjunctive.  Then $D_b$ is defined and
$\mn{ch}^{\text{sp}}_{\Omc}(\Amc_{C_0}) \vdash D_R(b)$ and (P2) yield
$\emptyset \models D_b \sqsubseteq D_R$.  It follows that there is a disjunct $K$ of
$D_R$ with $\emptyset \models E_b \sqsubseteq  K$.
Consequently, $\emptyset \models E_b \sqsubseteq D'_i$ for some~$i$. It follows that
when $\mn{ch}^{\text{sp}}_{\Omc}(\Amc_{C_0})$ was extended for
$b$, then 
$D'_i(b)$ has been added to $\mn{ch}^{\text{sp}}_{\Omc}(\Amc_{C_0})$.  If $D_i,D'_i$
come from Case~(ii), then this is an immediate contradiction.
Otherwise, non-applicability of Rules~2 and~3 yields
$\mn{ch}^{\text{sp}}_{\Omc}(\Amc_{C_0}) \vdash D'_i(b)$, a
contradiction to Case~(i).

Now assume that $b$ is not disjunctive. Since
$D_R(b) \in \mn{ch}^{\text{sp}}_{\Omc}(\Amc_{C_0})$, this implies that
$D_R$ has only a single disjunct $K$. This implies that
$D'_1 = \cdots = D'_k =K$. From
$\mn{ch}^{\text{sp}}_{\Omc}(\Amc_{C_0}) \vdash D_R(b)$ and
$D_R(b) \in \mn{ch}^{\text{sp}}_{\Omc}(\Amc_{C_0})$,
we thus obtain
$\mn{ch}^{\text{sp}}_{\Omc}(\Amc_{C_0}) \vdash D'_1(b)$ and
$D'_1(b) \in \mn{ch}^{\text{sp}}_{\Omc}(\Amc_{C_0})$, a
contradiction
in Case~(i) and~(ii), respectively.

\medskip

We now extend $\mn{ch}^{\text{sp}}_{\Omc}(\Amc_{C_0})$ for $E_a$ as 
described above. We observe the same property (P3) as above, the
proof is identical:
\begin{itemize}

\item[(P3)] if $\Amc^+ \vdash D(a)$ with $D \in \mn{Dis}(\Omc)$, then $\emptyset \models E_{a}
  \sqsubseteq D$.

\end{itemize}
We again show that no new rule applications are 
possible except applications of Rule~3 to original disjunctive
individuals. We only consider those cases explicitly for which the
arguments are not the same as above:
\begin{itemize}

\item Rule~1. Assume that $\Amc^+ \vdash C(c)$ and that $C \sqsubseteq
  D \in \Omc$.  We can use Property~(a) to show that the former
  implies $\mn{ch}^{\text{sp}}_{\Omc}(\Amc_{C_0}) \vdash C(c)$
  whenever $c \neq a$, and thus non-applicability of Rule~1 before the
  extension ensures $\mn{DNF}(D)(c) \in
  \mn{ch}^{\text{sp}}_{\Omc}(\Amc_{C_0}) \subseteq \Amc^+$. Now assume
  that $c=a$. From $\Amc^+ \vdash C(a)$ and (P3), we obtain $\emptyset
  \models E_{a} \sqsubseteq C$.  By definition of $D_{a}$, this
  implies $\emptyset \models E_{a} \sqsubseteq \mn{DNF}(D)$, and
  consequently $\mn{DNF}(a)$ is in $\Amc^+$.


\item Rule~4. Assume that $D_1(c)$ is in $\Amc^+$,
  $D_1 \in \mn{Dis}^-(\Omc)$, that $\Amc^+ \vdash D_2(c)$, and that
  $\Omc \models D_1 \sqcap D_2 \sqsubseteq D_3$.  First assume that
  $c \neq a$. Then
  $D_1(c) \in \mn{ch}^{\text{sp}}_{\Omc}(\Amc_{C_0})$. Moreover, we
  can use Property~(a) to show that $\Amc^+ \vdash D_2(c)$ implies
  $\mn{ch}^{\text{sp}}_{\Omc}(\Amc_{C_0}) \vdash D_2(c)$.  Thus,
  Rule~4 yields
  $D_2(c) \in \mn{ch}^{\text{sp}}_{\Omc}(\Amc_{C_0}) \subseteq
  \Amc^+$. Now assume that $c = a$.  Then
  $\emptyset \models E_{a} \sqsubseteq D_1$: if $D_1(a)$ is fresh in
  $\Amc^+$, then this is clear; otherwise,  otherwise, $D_1(a) \in
  \mn{ch}^{\text{sp}}_{\Omc}(\Amc_{C_0})$, which implies
  $\mn{ch}^{\text{sp}}_{\Omc}(\Amc_{C_0}) \vdash D_1(a)$ since
 $D_1 \in
  \mn{Dis}^-(\Omc)$, and thus
  (P1) yields $\emptyset \models E_{a} \sqsubseteq D_1$. Moreover,
  $\emptyset \models E_{a} \sqsubseteq D_2$ by (P3).  By definition of
  $D_{a}$, $\emptyset \models E_{a} \sqsubseteq D_2$ and thus $D_2(a)$
  is in $\Amc^+$.

\item Rule~5. Assume that $r(b,c), D_1(c) \in \Amc^+$ with
  $D_1(c)$ fresh and $D_1 \in \mn{Dis}^-(\Omc)$. Assume further that
  $\Omc \models \exists r . D_1 \sqsubseteq D_2$. Clearly,
  we must have $c=a$. Since $D_1(a) \in \Amc^+$ and
  $D_1 \in \mn{Dis}^-(\Omc)$, $\Amc^+ \vdash D_1(a)$.
  Thus (P2) yields $\emptyset \models E_a \sqsubseteq D_1$
and from Property~(b) we obtain $D_2(b) \in 
    \mn{ch}^{\text{sp}}_{\Omc}(\Amc_{C_0}) \subseteq \Amc^+$.

\end{itemize}
This finishes the extension of
$\mn{ch}^{\text{sp}}_{\Omc}(\Amc_{C_0})$ at $a$. It is not hard to
verify that (P1) holds for all original disjunctive $a$ for which the
extension has not yet been carried out, (P2) is preserved, and (P3)
holds for all original disjunctive $a$ for which the extension has
already been carried out. In particular, the extension of
$\mn{ch}^{\text{sp}}_{\Omc}(\Amc_{C_0})$ at $a$ does not invalidate
(P3) for original disjunctive $a'$ that have been treated earlier due
to Property~(a). We continue until the extension has taken place for
all original disjunctive $a$ and use \Emc to denote the ABox that is
obtained in the limit. No new rule applications are possible with the
exception of applications of Rule~3 to original disjunctive
individuals.

\smallskip

Recall that we aim to construct a model \Imc of \Omc with an element
$d$ such that
$d \in C_0^\Imc \setminus (\mn{Dis}^\EL_\Omc(C_0))^\Imc$. We are going
to start from $\Emc^-$, that is, $\Emc$ restricted to role assertions
and atomic concept assertions, viewed as an interpretation. The
resulting interpretation \Imc, however, need not be a model of \Omc,
for two reasons. First, new applications of Rule~3 to original
disjunctive individuals $a$ are possible which means that there might
be assertions
$\exists r . C(a) \in \Emc$ such
that $a \notin (\exists r . C)^\Imc$, and this in turn means that some
CIs in \Omc might not be satisfied. And second, we have chosen
disjuncts $E_a$ of the disjunctions $D_a$ at original disjunctive
individuals to ensure that all disjunctions are satisfied at original
individuals, but we have not ensured the same at anonymous
individuals. We thus modify the initial \Imc in two ways, which both
involve grafting additional tree-shaped interpretations that we select
in what follows. We first observe that
\begin{itemize}

\item[($*$)] If
  $C(a) \in \Emc$ with
  $a$ original and disjunctive, then $\emptyset \models E_a \sqsubseteq C$.

\end{itemize}
To see this, first assume that
$\exists r .C(a) \in \mn{ch}^{\text{sp}}_{\Omc}(\Amc_{C_0})$ already
before the extension to \Emc. Then non-applicability of Rules~2 and~3 implies
$\mn{ch}^{\text{sp}}_{\Omc}(\Amc_{C_0}) \vdash C(a)$ and (P3) yields
$\emptyset \models E_a \sqsubseteq C$. Otherwise,
$\emptyset \models E_a \sqsubseteq C$ by construction of \Emc.

\smallskip

By ($*$) and the (semantic!) definition of $D_a$ of which $E_a$ is a
disjunct, we find for each
$\exists r .C(a) \in \Emc$ with $a$
original and disjunctive, a tree model $\Imc_{\exists r . C(a)}$ of
\Omc with root $d$ such that $d \in C^{\Imc_{\exists r . C(a)}}$
and
$d \in F^{\Imc_{\exists r . C(a)}}$
implies $\emptyset \models E_a \sqsubseteq \exists r .F$ for all
$\exists r . F \in \mn{sub}(\Omc)$. 

Let $\Gamma$ denote the set of individuals $a$ in
$\mn{ch}^{\text{sp}}_{\Omc}(\Amc_{C_0})$ that are anonymous and marked
and whose predecessor is anonymous and unmarked.\footnote{We work
  here with the anonymous part, which is identical in
  $\mn{ch}^{\text{sp}}_{\Omc}(\Amc_{C_0})$ and in \Emc.}
  Let $a \in \Gamma$
have been introduced for $C_a$ and let $r(b,a)$ be the unique
assertion of this form in $\mn{ch}^{\text{sp}}_{\Omc}(\Amc_{C_0})$.
Since Rule~5 is not applicable, $\mn{Dis}_\Omc(\exists r . C_a)(b)
\in \mn{ch}^{\text{sp}}_{\Omc}(\Amc_{C_0})$ and since $b$ is anonymous
and not marked, $\mn{Dis}_\Omc(\exists r . C_a)\notin
\mn{Dis}^-(\Omc)$. Furthermore, $\mn{Dis}_\Omc(\exists r . C_a)$
is not empty since we assume that
$\mn{ch}^{\text{sp}}_{\Omc}(\Amc_{C_0})$
contains no assertion of the form $\bot(b)$.
Consequently, we find a tree model $\Imc_a$ of \Omc with
root $a \in C_a^{\Imc_a}$ such that, in the extension $\Jmc$ of
$\Imc_a$ obtained by adding an $r$-predecessor $b$ to the root $a$ of
$\Imc_a$, we have $b \in (\exists r . C)^{\Jmc}$ iff $\Omc \models C_a
\sqsubseteq C$ for all $\exists r . C \in
\mn{sub}(\Omc)$. 
  
  \medskip

  Construct an interpretation $\Imc$ as follows:
  \begin{itemize}

  \item start with $\Emc^-$
    viewed as an 
    interpretation;

  \item for each 
    $\exists r .C(a) \in \Emc$ with 
    $a$ original and disjunctive, disjointly add the interpretation
    $\Imc_{\exists r . C(a)}$ with root $d$ and extend $r^\Imc$ with
    $(a,d)$;

  \item for each $a \in \Gamma$, replace the subtree rooted at $a$
    with $\Imc_a$.

  \end{itemize}
  We next observe the following.
  \\[2mm]
  {\bf Claim 1.} Let $a \in \Delta^\Imc$ be an individual of
  $\mn{ch}^{\text{sp}}_{\Omc}(\Amc_{C_0})$ and let
  $C \in \mn{sub}(\Omc)$ be an \EL concept. Then $a \in C^\Imc$ implies
  \begin{enumerate}

  \item 
 $\emptyset \models E_a \sqsubseteq C$
    if $a$ is original and disjunctive, and

  \item $\mn{ch}^{\text{sp}}_{\Omc}(\Amc_{C_0}) \vdash C(a)$ otherwise.

    \end{enumerate}
    The proof is by induction on the structure of $C$. In the
    induction start, $C=A$ is a concept name. Since $a \in A^\Imc$, we
    must have $A(a) \in \Emc$ and thus $\Emc \vdash A(a)$.  If $a$ is
    original and disjunctive, then (P3) yields
    $\emptyset \models E_a \sqsubseteq C$ as required. If this is
    not the case, then $A(a) \in \Emc$ implies $A(a) \in
    \mn{ch}^{\text{sp}}_{\Omc}(\Amc_{C_0})$, thus
    $\mn{ch}^{\text{sp}}_{\Omc}(\Amc_{C_0}) \vdash A(a)$ as required.

The case $C=C_1 \sqcap C_2$ is
straightforward using the semantics and induction hypothesis. Details
are left to the reader.

It thus remains to deal with the case $C= \exists r . C_1$. Then
$a \in C^\Imc$ implies that there is a $d \in C_1^\Imc$ with
$(a,d) \in r^\Imc$. We distinguish several cases. First assume that
$d$ is an individual from $\mn{ch}^{\text{sp}}_{\Omc}(\Amc_{C_0})$
that is not in $\Gamma$.  We have the following subcase:
  \begin{enumerate}

  \item $a$ and $d$ are original and disjunctive.

    The induction hypothesis yields $\emptyset \models E_d
    \sqsubseteq C_1$. Thus Condition~(a) from the extension
    step ensures that $\mn{ch}^{\text{sp}}_{\Omc}(\Amc_{C_0}) \vdash
    \exists r . C_1(a)$. Property (P1) yields $\emptyset \models E_a
    \sqsubseteq \exists r . C_1$, as required.

  \item $a$ is original and disjunctive and $d$ is not.

    The induction hypothesis yields $\mn{ch}^{\text{sp}}_{\Omc}(\Amc_{C_0}) \vdash
    C_1(d)$ and the
    construction of \Imc yields $r(a,d) \in
    \mn{ch}^{\text{sp}}_{\Omc}(\Amc_{C_0})$,
    thus $\mn{ch}^{\text{sp}}_{\Omc}(\Amc_{C_0}) \vdash \exists
    r.C_1(a)$. Property (P1) yields $\emptyset \models E_a
    \sqsubseteq \exists r . C_1$, as required.
    
  \item $d$ is original and disjunctive and $a$ is not.

    The induction hypothesis yields $\emptyset \models E_d
    \sqsubseteq C_1$. Thus Condition~(a) from the extension
    step ensures that $\mn{ch}^{\text{sp}}_{\Omc}(\Amc_{C_0}) \vdash
    \exists r . C_1(a)$, as required.

  \item neither $a$ nor $d$ are original and disjunctive.
    
    The induction hypothesis yields $\mn{ch}^{\text{sp}}_{\Omc}(\Amc_{C_0}) \vdash
    C_1(d)$ and the
    construction of \Imc yields $r(a,d) \in
    \mn{ch}^{\text{sp}}_{\Omc}(\Amc_{C_0})$,
    thus $\mn{ch}^{\text{sp}}_{\Omc}(\Amc_{C_0}) \vdash \exists
    r.C_1(a)$, as required. 

  \end{enumerate}
  The following cases remain:

  \begin{enumerate}

\setcounter{enumi}{4}
    
  \item $a$ is original and disjunctive and there is an
    $\exists r . E(a) \in \mn{ch}^{\text{sp}}_{\Omc}(\Amc_{C_0})$
    such that $d$ is the root of $\Imc_{\exists r .E(a)}$.

    By choice of $\Imc_{\exists r .E(a)}$, this implies
    $\emptyset \models E_a \sqsubseteq \exists r . C_1$, as required.
    
  \item \emph{$d \in \Gamma$} and thus the root of $\Imc_d$.

    Then $d \in C_1^{\Imc_d}$. By choice of $\Imc_d$ and since
    $\exists r . C_1 \in \mn{sub}(\Omc)$, we have
    $\Omc \models C_d \sqsubseteq C_1$ where $C_d$ is the concept that
    $d$ was introduced for. We moreover have
    $C_d(d) \in \mn{ch}^{\text{sp}}_{\Omc}(\Amc_{C_0})$ and by
    non-applicability of Rules~2 and~3 also
    $\mn{ch}^{\text{sp}}_{\Omc}(\Amc_{C_0}) \vdash C_d(d)$. Since $d$
    is in $\Gamma$, it is marked. Thus
    $D(d) \in \mn{ch}^{\text{sp}}_{\Omc}(\Amc_{C_0})$ for some
    $D \in \mn{Dis}^-(\Omc)$. Clearly,
    $\Omc \models D \sqcap C_d \sqsubseteq C_1$. We can thus
    invoke Rule~4 with
    $D_1 =D$, $D_2=C_d$, and $D_3=C_1$ to yield
    $C_1(d) \in \mn{ch}^{\text{sp}}_{\Omc}(\Amc_{C_0})$. Rules~2 and~3
    thus ensure that
    $\mn{ch}^{\text{sp}}_{\Omc}(\Amc_{C_0}) \vdash C_1(d)$.  From
    $(a,d) \in r^\Imc$, we obtain
    $r(a,d) \in \mn{ch}^{\text{sp}}_{\Omc}(\Amc_{C_0})$, and thus
    $\mn{ch}^{\text{sp}}_{\Omc}(\Amc_{C_0}) \vdash \exists r.C_1(a)$
    as required (since $a$ cannot be original).

  \end{enumerate}
  This finishes the proof of the claim.

  \medskip
  
    \noindent
    {\bf Claim 2.}  Let $a \in \Delta^\Imc$ be an individual of
  $\mn{ch}^{\text{sp}}_{\Omc}(\Amc_{C_0})$ and let $C \in
  \mn{sub}(\Omc)$. Then $\mn{ch}^{\text{sp}}_{\Omc}(\Amc_{C_0}) \vdash C(a)$
 implies
  $a \in C^\Imc$.
  \\[2mm]
  The proof is by induction on the structure of $C$. The case that $C$
  is a concept name is clear by construction of \Imc. The case that
  $C=C_1 \sqcap C_2$ and
  $C=\exists r . C_1$ are straightforward using
  the fact that Rules~2 and~3 are not applicable and the induction
  hypothesis. It remains to deal with the case $C= C_1 \sqcup
  C_2$. Then
  $C_1 \sqcup C_2(a) \in \mn{ch}^{\text{sp}}_{\Omc}(\Amc_{C_0})$ and
  thus $a$ is disjunctive and $C_1 \sqcup C_2$ is a conjunct of every
  disjunct of $D_a$, including the disjunct $E_a$ chosen for $a$
  during the extension of $\mn{ch}^{\text{sp}}_{\Omc}(\Amc_{C_0})$.
  By definition of $D_a$, it follows that some $C_i$, $i \in \{1,2\}$,
  is also a conjunct of $E_a$. Thus $C_i(a)$ has been added in the
  extension of  $\mn{ch}^{\text{sp}}_{\Omc}(\Amc_{C_0})$ and it
  remains
  to apply the induction hypothesis. This finishes the proof of Claim~4

  \medskip

  Note that \Imc is a tree interpretation with root $a_0$. By
  construction of \Imc, it is clear that $a_0 \in C_0^\Imc$. We next
  show that \Imc is a model of \Omc.

\smallskip 

Let $C \sqsubseteq D \in \Omc$ and let $d \in C^\Imc$. If $d$ is in
the domain of some interpretation $\Imc_{\exists r . C(a)}$ or
$\Imc_a$, then it follows from the construction of \Imc and the fact that
all interpretations $\Imc_{\exists r . C(a)}$ and $\Imc_a$ are models
of \Omc that $d \in D^\Imc$. Thus let $a$ be an individual of
$\mn{ch}^{\text{sp}}_{\Omc}(\Amc_{C_0})$.   We distinguish two cases.

First assume that $a$ is original and disjunctive.  Then Point~1 of Claim~1 yields
$\emptyset \models E_{a} \sqsubseteq C$ and as a consequence
we have $\emptyset \models E_{a} \sqsubseteq D$ which implies
$F(a) \in \Emc$ for all top-level conjuncts $F$ of $D$. If $F$ is a
concept name, then this yields $a \in F^\Imc$ by construction of
\Imc. If $F$ takes the form $\exists r . G$, then the addition of
$\Imc_{\exists r . G}$ ensures that $a \in (\exists r . G)^\Imc$.
As a consequence, $a \in D^\Imc$.

Now assume that $a$ is not original or not disjunctive.  Then Point~2
of Claim~1 yields $\mn{ch}^{\text{sp}}_{\Omc}(\Amc_{C_0}) \vdash
C(a)$. Non-applicability of Rule~1 of the chase yields $\mn{DNF}(D)(a)
\in \mn{ch}^{\text{sp}}_{\Omc}(\Amc_{C_0})$ and non-applicability of
Rules~2 and~3 yields $\mn{ch}^{\text{sp}}_{\Omc}(\Amc_{C_0}) \vdash
\mn{DNF}(D)(a)$. By Claim~2, $a \in D^\Imc$.

  \smallskip 

  It remains to show that $a_0 \notin \mn{Dis}^\EL_\Omc(C_0)^\Imc$.
  Assume to the contrary that $a_0 \in \mn{Dis}^\EL_\Omc(C_0)^\Imc$.
  Then there is a disjunct $K$ of $\mn{Dis}^\EL_\Omc(C_0)^\Imc$ such
  that $a_0 \in C^\Imc$ for every conjunct $C$ of~$K$.  We distinguish two cases.

  First assume that $a_0$ is disjunctive.  By Point~1 of Claim~1,
  $\emptyset \models E_{a_0} \sqsubseteq C$ for all conjuncts~$C$
  of~$K$.  Thus
  $\emptyset \models E_{a_0} \sqsubseteq \mn{Dis}^\EL_\Omc(C_0)$, in
  contradiction to our choice of $E_{a_0}$.
  
  Now assume that $a_0$ is not disjunctive. By Point~2 of Claim~1,
  $\mn{ch}^{\text{sp}}_{\Omc}(\Amc_{C_0}) \vdash C(a_0)$ for all
  disjuncts $C$ of $K$. By Lemma~\ref{lem:chsound},
  $\Omc \models C_0 \sqsubseteq C$ for all such $C$.
  Since \Imc is a model of \Omc, this implies that
  $\mn{Dis}^\EL_\Omc(C_0)$ has only the disjunct $K$.
  We have thus shown that $\mn{ch}^{\text{sp}}_{\Omc}(\Amc_{C_0}) \vdash
  \mn{Dis}^\EL_\Omc(C_0)$, a contraction.
\end{proof}

\section{Proof of Theorem~\ref{thm:smallercomplete}}

Let $\mn{1Dis}(\Omc_S)$ denote the set of disjunctions of concepts
from $\mn{sub}^-(\Omc_S)$, without repetition, of which there are
clearly only single exponentially many. For each $D \in
\mn{Dis}(\Omc_S)$, let $K_D$ be the set of disjunctions $D' \in
\mn{1Dis}(\Omc_S)$ such that $D'$ is a conjunct of the result of
converting $D$ viewed as a DNF formula (in which all concept names and
concepts $\exists r . E$ serve as propositional variables) into
KNF. We then have $\emptyset \models D \equiv \bigsqcap K_D$.  For
$\ell \in \mathbb{N} \cup \{ \omega \}$, let the \EL ontology
$\widehat\Omc^\ell_T$ be defined as $\Omc^\ell_T$ in
Figure~\ref{fig:second}, but with every occurrence of a concept name
$X_D$ replaced by $\bigsqcap_{C \in K_D} Y_C$.

\smallskip

Theorem~\ref{thm:smallercomplete} is a consequence of
the following.
\begin{lemma}
\label{lem:singleexp}
Let $\ell \in \mathbb{N} \cup \{ \omega \}$. For all \EL concepts $C_0,D_0$ over $\mn{sig}(\Omc_S)$, $\Omc^\ell_T
\models C \sqsubseteq D$ iff $\widehat \Omc^\ell_T \models C \sqsubseteq D$.
\end{lemma}
\noindent
\begin{proof}\ It suffices to show that for every model $\Imc$ of
  $\Omc^\ell_T$ there is a model $\widehat \Imc$ of $\widehat
  \Omc^\ell_T$ such that the restrictions of \Imc and $\widehat \Imc$
  to the symbols of $\mn{sig}(\Omc_S)$ are identical and vice versa.

  We start with the easier direction. Thus let $\widehat \Imc$ be a
  model of~$\widehat \Omc^\ell_T$. Let \Imc be defined like $\widehat
  \Imc$ except that $X_D^\Imc=\bigsqcap_{D' \in K_D} Y_{D'}^{\widehat
    \Imc}$.  Observe that the concept names $X_D$ are not in
  $\mn{sig}(\Omc_S)$ and thus the restrictions of \Imc and $\widehat
  \Imc$ to the symbols of $\mn{sig}(\Omc_S)$ are identical, as
  required. It is not straightforward to verify that \Imc satisfies
  every CI $C_1 \sqsubseteq C_2 \in \Omc^\ell_T$ given that $\widehat
  \Omc^\ell_T$ contains a CI $C'_1 \sqsubseteq C'_2$ such that $C'_i$
  can be obtained from $C'$ by replacing each $X_D$ with
  $\bigsqcap_{D' \in K_D} Y_{D'}$ and that $\widehat \Imc$ satisfies
  $C'_1 \sqsubseteq C'_2$.

  For the converse direction, let \Imc be a model of $\Omc^\ell_T$. We
  cannot define a corresponding $\widehat \Imc$ of $\widehat
  \Omc^\ell_T$ by setting $\bigsqcap_{D' \in K_D} Y_{D'}^{\widehat
    \Imc}=X_D^\Imc$ because we need to interpret individual concept
  names $Y_{D'}$ rather than conjunctions thereof. To achieve this, we
  resort to semantic disjunctions $\mn{Dis}_{\Omc_S}(D)$. In fact, we
    define $\widehat \Imc$ to be like \Imc except that 
    $$
    Y_D^{\widehat
      \Imc}=\big ({\mn{Dis}_{\Omc_S}(D)}^\uparrow
    \sqcap
\big (
\bigsqcup_{D' \in \mn{Dis}^-(\Omc_S) \mid D \in K_{D'}} X_{D'}
\big)\big)^\Imc
    $$
    for every $D \in \mn{1Dis}(\Omc_S)$. It remains to show that
    $\widehat \Imc$ is a model of of $\widehat \Omc^\ell_T$. We can
    argue exactly as in the converse direction if we know that
    $\bigsqcap_{D' \in K_D} Y_{D'}^{\widehat \Imc}=X_D^\Imc$ for all
    $D \in \mn{Dis}^-(\Omc_S)$. By definition of $\widehat \Imc$, this
    amounts to showing that 
    $$\big ( \!\!\!\! \bigsqcap_{D' \in K_D}
    \mn{Dis}_{\Omc_S}(D')^\uparrow \sqcap
\big (
\bigsqcup_{D'' \in \mn{Dis}^-(\Omc_S) \mid D' \in K_{D''}} \!\!\! X_{D''}
\big) \big)^\Imc = X_D^\Imc.
    $$
    For the `$\supseteq$' direction, we note that $\Omc_S \models D
    \sqsubseteq \mn{Dis}_{\Omc_S}(D')^\uparrow$ for every $D' \in K_D$
    since $\emptyset \models D \sqsubseteq D'$ and due to the
    definition of $\mn{Dis}_{\Omc_S}(D')$. Thus $\Omc^\ell_T$ contains
    the CI $X_D \sqcap X_D \sqsubseteq
    {\mn{Dis}_{\Omc_S}(D')}^\uparrow$ and consequently $\Omc^\ell_T
    \models X_D \sqsubseteq \mn{Dis}_{\Omc_S}(D') \sqcap X_D$ for
    every $D' \in K_D$. It suffices to recall that $\Imc$ is a model
    of~$\Omc^\ell_T$.

    For the `$\subseteq$' direction, 
    assume that
    $$
    d \in (\bigsqcap_{D' \in K_D}
    \mn{Dis}_{\Omc_S}(D')^\uparrow \sqcap X_{D''_0})^\Imc
    $$
    for some $D''_0 \in \mn{Dis}^-(\Omc_S)$. Let $K_D=\{D_1,\dots,D_k
    \}$ and let $D'_1=\mn{Dis}_{\Omc_S}(D''_0 \sqcap D_1)$.

    We first argue that $d \in ({D'_1}^\uparrow)^\Imc$. This follows
    from $\Omc_S \models D''_0 \sqcap D_1 \sqsubseteq D'_1$ and the fact
    that, thus, $\Omc^\ell_T$ contains the CI $X_{D''_0} \sqcap
    D_1^\uparrow \sqsubseteq {D'_1}^\uparrow$ and since $\Imc$ is a model of
    $\Omc^\ell_T$. 

    In the very special case that $D'_1$ consists of a single disjunct
    that contains all concepts from $\mn{sub}^-(\Omc_S)$ as conjuncts,
    we actually have $\Omc_S \models D''_0 \sqcap D_1 \sqsubseteq D$,
    and thus we can argue as above that $d \in (D^\uparrow)^\Imc$ and
    are done since $D^\uparrow= X_D$ given that $D \in
    \mn{Dis}^-(\Omc_S)$. 

    Otherwise, we can find a $D''_1 \in \mn{Dis}(\Omc_S)$ with at
    least two disjuncts such that $\emptyset \models D'_1 \equiv
    D''_1$. Let $D'_2=\mn{Dis}_{\Omc_S}(D''_1 \sqcap D_2)$. As in the
    case of $D''_1$, we can show that $d \in ({D'_2}^\uparrow)^\Imc$.
    We can repeat this until we have shown that $d \in
    ({D'_k}^\uparrow)^\Imc$, $D'_k=\mn{Dis}_{\Omc_S}(D''_{k-1} \sqcap
    D_k)$. Again, we are done if $D'_k$ consists of a single disjunct
    that contains all concepts from $\mn{sub}^-(\Omc_S)$ as conjuncts.
    Otherwise, we can find a a $D''_k \in \mn{Dis}(\Omc_S)$ with at
    least two disjuncts such that $\emptyset \models D'_k \equiv
    D''_1$.

    By construction of $D'_k$ and choice of $D''_k$, we must have
    $\Omc_S \models D''_k \sqsubseteq D$. Thus $\Omc^\ell_T$ contains
    the CI $X_{D''_k} \sqcap {D''_k}^\uparrow \sqsubseteq D^\uparrow$.
    It follows that $d \in X_D^\Imc$.
\end{proof}

\section{Proof of Theorem~\ref{thm:acyclicbetter}}

Theorem~\ref{thm:corr} follows from the following lemma. 
\begin{lemma}
\label{lem:acyclic}
Let \Omc be an acyclic $\ELU$ ontology and let
$$
C=A_1 \sqcap \cdots \sqcap A_n \sqcap \exists r_1 . E_1 \sqcap \cdots
\sqcap \exists r_m . E_m, \quad D=\exists r.E
$$
be \EL concepts such that $\Omc\models C \sqsubseteq D$ and 
there does not exist any $i \leq m$ with $r_{i}=r$ and $\Omc \models E_{i}\sqsubseteq E$.
Then there exists $i\leq n$ with $\Omc \models A_{i} \sqsubseteq D$.
\end{lemma}
\begin{proof} \
Assume the lemma does not hold. Take tree shaped interpretations $\Imc_{i}$, $1\leq i \leq n$, with root $a_{i}$
and $\Jmc_{i}$, $1\leq i\leq m$, with root $b_{i}$ such that all $\Imc_{i},\Jmc_{i}$ are models of $\Omc$
and 
\begin{itemize}
\item $a_{i}\in A_{i}^{\Imc_{i}}$ and $a_{i}\not\in D^{\Imc_{i}}$, for all $1\leq i \leq n$;
\item $b_{i} \in E_{i}^{\Jmc_{i}}$ and $b_{i}\not\in E^{\Imc_{i}}$, for all $1\leq i \leq m$.
\end{itemize}
Construct a model $\Imc$ by taking the disjoint union of all $\Imc_{i},\Jmc_{i}$ and
then identifying all $a_{i}$, $1\leq i \leq m$, to a single node $a$
and adding $(a,b_{i})$ to the interpretation of $r_{i}$.
Next define $\Imc'$ by adding in $\Imc$, recursively, $a$ to the interpretation of a concept name $A$
if there exists $C$ such that $A\equiv C \in \Omc$ and $a\in C^{\Imc}$. We claim that $\Imc'$ 
is a model of $\Omc$ and $a\not\in D^{\Imc}$. The latter holds by definition. For the former,
consider some $A'\equiv C'\in \Omc'$ for which $a$ has not been added to the interpretation of $A'$
in the step above (the remaining CIs are trivially true in $\Imc'$). Then it only remains to
check that $a\in A'^{\Imc}$ implies $a\in C'^{\Imc}$, but this follows by construction again.
\end{proof}

\section{Proof of Theorems~\ref{thm:corrbot} and~\ref{thm:corrminus}}

The \EL chase introduced in Appendix~\ref{app:prelims} can be extended
to $\EL_\bot$ in a straightforward way. Recall that we assume $\bot$ to
occur only in CIs of the form $C \sqsubseteq \bot$. The $\EL_\bot$
chase is defined exactly as the \EL chase. In particular, it also
treats CIs of the form $C \sqsubseteq \bot$, adding $\bot(a)$ to an
ABox \Amc when $\Amc \models C(a)$, and thus producing $\EL_\bot$
extended ABoxes. We write $\mn{ch}_\Omc(\Amc_C) \models \bot$ if
there is some $a$ with $\bot(a) \in \mn{ch}_\Omc(\Amc_C)$. The
correctness of the chase now reads as follows.
\begin{lemma}
\label{lem:chasebotcorr}
Let \Omc be an $\EL_\bot$ ontology and let $C,D$ be
$\EL_\bot$ concepts.  Then $\Omc \models C \sqsubseteq D$ iff
$\mn{ch}_\Omc(\Amc_C) \models D(a_0)$ or
$\mn{ch}_\Omc(\Amc_C) \models \bot$.
\end{lemma}
Based on Lemma~\ref{lem:chasebotcorr}, we can prove the soundness of
the approximation. The proof is essentially identical to that of
Lemma~\ref{lem:sound}, that is, to the correctness of the
approximation in the \ELU-to-\EL case. We omit details.
\begin{lemma} \label{lem:soundbot} $\Omc^\omega_T \models C_0 \sqsubseteq
  D_0$ implies $\Omc_S \models C_0 \sqsubseteq D_0$ for all \EL
  concepts $C_0,D_0$ over $\mn{sig}(\Omc_S)$.
\end{lemma}
Now for completeness. Recall that we have established the central
Lemma~\ref{lem:compaux} already for the case where $\Omc_S$ is an
$\ELU_\bot$ ontology. The same is true for Lemma~\ref{lem:smallC}.  
%
\begin{restatable}{lemma}{lemcompbot} \label{lem:compbot} Let $\ell \in
  \mathbb{N} \cup \{ \omega \}$. Then $\Omc_S \models C_0 \sqsubseteq
  D_0$ implies $\Omc^\ell_T \models C_0 \sqsubseteq D_0$ for all \EL
  concepts $C_0,D_0$ over $\mn{sig}(\Omc_S)$ such that the role depth
  of $D_0$ is bounded by $\ell$.
\end{restatable}

\noindent
\begin{proof}\ Assume that $\Omc_S \models C_0 \sqsubseteq D_0$ with
  $C_0,D_0$ $\EL$ concepts over $\mn{sig}(\Omc_S)$ such that the role
  depth of $D_0$ is bounded by~$\ell$. If $C_0$ contains $\bot$, then
  clearly $\Omc^\ell_T \models C_0 \sqsubseteq D_0$. If $D_0$ contains
  $\bot$, then it is equivalent to $\bot$. We can thus assume that
  $C_0$ is an \EL concept and it suffices to consider the cases where
  $D_0$ is $\bot$, a concept name, or of the form $\exists r . E_0$.

  We start with the case $D_0=\bot$. Then $\Omc_S \models C_0
  \sqsubseteq D_0$ implies that $\mn{Dis}^\EL_{\Omc_S}(C_0)$ is $\bot$
  (that is, it is the empty disjunction), and consequently
  Lemma~\ref{lem:compaux} yields $\Omc^\ell_T \models C_0 \sqsubseteq
  \bot$ as required.

  Now let $D_0=A$.  Clearly, $\Omc_S \models C_0 \sqsubseteq A$
  implies $\Omc_S \models \mn{Dis}^\EL_{\Omc_S}(C_0) \sqsubseteq A$.
  It thus follows from Lemma~\ref{lem:compaux} that $O_T^\ell \models
  C_0 \sqsubseteq A$ (see proof of Lemma~\ref{lem:compbot} for 
  details).

  The case where $D_0=\exists r . E_0$ is a consequence of the
  following claim. For each $a \in \mn{Ind}(\Amc_{C_0})$, we write
  $C^a_0$ as an abbreviation for $C^a_{\Amc_{C_0}}$.
  \\[2mm]
  {\bf Claim.}  For all $a \in \mn{Ind}(\Amc_{C_0})$ and \EL concepts
  $\exists r . E$ of depth $\ell-\mn{depth}(a)$, $\Omc_S \models
  C^a_0 \sqsubseteq \exists r . E$ implies $\Omc^\ell_T
  \models C^a_0 \sqsubseteq \exists r . E$.
  \\[2mm]
  \emph{Proof of claim}. The proof is by induction on the codepth of~$a$. 

  \smallskip 
  \noindent 
  \emph{Induction start}. Then $a$ is a leaf in $\Amc_{C_0}$ and thus
  $C^a_0$ does not have any top-level conjuncts of the form $\exists r
  . E'$.  Thus Condition~1 from Figure~\ref{fig:third} is satisfied
  for $F=C^a_0$. Consequently, $\Omc^\ell_T$ contains the CI
  $C^a_0 \sqsubseteq \exists r . E$ and we are done.

  \smallskip 
  \noindent 
  \emph{Induction step.} Then $a$ is a non-leaf in $\Amc_{C_0}$. We
  distinguish two cases.

  \medskip
  \noindent
  \emph{Case~1}. There is a top-level conjunct $\exists r . E'$ in
  $C^a_0$ such that $\Omc_\Smc \models E' \sqsubseteq
  E$. Then $a$ has an $r$-successor $b$ in $\Amc_{C_0}$ such that
  $C^b_0 = E'$.  Let
  $$
  E=A_1 \sqcap \cdots \sqcap A_{n} \sqcap \exists r_1 . E_1 \sqcap
  \cdots \sqcap \exists r_{m} . E_{m}.
  $$
  Since we have already shown Lemma~\ref{lem:compbot} for the case where
  $D_0$ is a concept name, we obtain $\Omc_T^\ell \models
  C^b_0 \sqsubseteq A_i$ for $1 \leq i \leq n$.  From the
  induction hypothesis, we further obtain $\Omc_T^\ell \models
  C^b_0 \sqsubseteq \exists r_i . E_i$ for $1 \leq i \leq
  m$. Thus $\Omc_T^\ell \models C^b_0 \sqsubseteq E$ and
  consequently $\Omc_T^\ell \models C^a_0 \sqsubseteq
  \exists r . E$ as required.
  
  \medskip
  \noindent
  \emph{Case~2}. There is no top-level conjunct $\exists r . E'$ in
  $C^a_0$ such that $\Omc_\Smc \models E' \sqsubseteq E$.  Then
  Condition~1 from Figure~\ref{fig:third} is satisfied for $F=C^a_0$.
  Let $\Amc$ be the ditree-shaped subABox of $\Amc_{C_0}$ rooted at
  $a$ and let $\Amc^\pm$ be the extended ABox obtained from
  $\Amc|_{k}$, with $k$ the depth of $\exists r . E$, by adding
  $\mn{Dis}^\EL_{\Omc_S}(C^c_\Amc)^\uparrow(c)$ whenever $c$ is a leaf
  in $\Amc|_{k}$. Applying Lemma~\ref{lem:smallC} to \Amc and
  $\Amc^\pm$ and with $\exists r .E$ in place of $\exists r . C$, we
  obtain $\Omc_S,\Amc^\pm \models \exists r . E(a)$.  Let $C^\pm$ be
  $\Amc^\pm$ viewed as an \EL concept decorated with disjunctions from
  $\mn{Dis}(\Omc_S)$ at leaves. Since there is no top-level conjunct
  $\exists r . E'$ in $C^a_0$ such that $\Omc_\Smc \models E'
  \sqsubseteq E$ there is no top-level conjunct $\exists r . E'$ in
  $C^\pm$ such that $\Omc_\Smc \models E' \sqsubseteq E$ either: if
  this was the case with $\exists r .E'$ corresponding to the
  successor $r(a,b)$ of $a$ in $\Amc^\pm$, then we can apply
  Lemma~\ref{lem:smallC} to the subABox of \Amc rooted at $b$ and the
  subABox of $\Amc^\pm$ rooted at $b$ to obtain $\Omc_\Smc,\Amc \models
  \exists r . E (b)$ and thus $b$ in \Amc corresponds to
  a top-level conjunct
  $\exists r . E'$ in $C^a_0$ with $\Omc_\Smc \models E'
  \sqsubseteq E$.
\end{proof}

\section{More Optimization for Figure~\ref{fig:third}}

\begin{figure}[t]
   \begin{boxedminipage}{\columnwidth}
     \vspace*{-2mm}
  \begin{center}
  $$
  \begin{array}{rcll}
    C &\sqsubseteq& E^\uparrow & \text{if } C \sqsubseteq E
  \in \Omc_S \\[0.5mm]
  X_{D} \sqcap D_1^\uparrow &\sqsubseteq&
  D_2^\uparrow & \text{if } \Omc_S \models D \sqcap D_1 \sqsubseteq
  D_2\\[0.5mm]
 \exists r . X_{D} &\sqsubseteq& D_1^\uparrow & \text{if }
   \Omc_S \models \exists r . D \sqsubseteq D_1 \\[0.5mm]
   X_D & \sqsubseteq& \exists r . D_1^\uparrow & \text{if } \Omc_S \models D \sqsubseteq 
  \exists r . D_1\\[0.5mm]
   F^\uparrow &\sqsubseteq& \exists r . G & \text{if } \Omc_S \models F \sqsubseteq
  \exists r . G 
  \end{array}
  $$
  \end{center}
  where in the last line $F$ is an \EL concept over $\mn{sig}(\Omc_S)$
  decorated with disjunctions from $\mn{Dis}(\Omc_S)$ at leaves and
  $G$ is an \EL concept over $\mn{sig}(\Omc_S)$ such that 
    \begin{enumerate}

    \item $F$ has no top-level conjunct $\exists r . F'$ s.t.\ $\Omc_S
      \models F' \sqsubseteq G$;



    \item it is not the case that $\mn{Dis}_{\Omc_S}(F)$ has at least two
      disjuncts and $\Omc_S \models \mn{Dis}_{\Omc_S}(F) \sqsubseteq
      \exists r . G$;

    \item   $\mn{depth}(F) \leq \mn{depth}(G) < \ell$.

    \end{enumerate}
  \end{boxedminipage}
  \caption{Optimized $\ell$-bounded $\EL_\bot$ approximation $\Omc^\ell_T$.}
\label{fig:thirdb}
\vspace*{-3mm}
\end{figure}
A further optimization of the approximation from
Figure~\ref{fig:third} is shown in Figure~\ref{fig:thirdb} where
$\mn{Dis}_{\Omc_S}(C_0)$ is defined just like
$\mn{Dis}^\EL_{\Omc_S}(C_0)$ except that the disjunctions and
conjunctions are based on all concepts from $\mn{sub}^-(\Omc_S)$
rather than only those formulated in \EL.  Compared to
Figure~\ref{fig:third}, the second last concept inclusion and
Condition~2 have been added, with the aim of invoking the expensive
bottommost concept inclusion less often.
\begin{exmp}
  Consider the following variation of the \ELU ontology in
  Proposition~\ref{ex:1}:
  $$
  \begin{array}{r@{}rcl@{}l}
    \Omc_S = \{ & A &\sqsubseteq& B_1 \sqcup B_2,\\[0.5mm]
& \exists r . B_i &\sqsubseteq& B_i, & \text{for } i \in \{1,2\}\\[0.5mm]
&    B_i \sqcap A' &\sqsubseteq& \exists r . M& \text{for } i \in \{1,2\} \ \}.
  \end{array}
  $$
  The approximation $\Omc^\omega_T$ in Figure~\ref{fig:third} would
  contain the CI
  \begin{equation*}
  A' \sqcap \exists r^n . A \sqsubseteq \exists r . M 
\tag{$\dagger$}
  \end{equation*}
  for all $n \geq 1$. However, $\mn{Dis}_{\Omc_S}(A' \sqcap \exists
  r^n . A)$ is 
  $$
   (A' \sqcap B_1 \sqcap \exists r . B_1) \sqcup
      (A' \sqcap B_2 \sqcap \exists r . B_2) 
  $$
  and we have $\Omc_S \models \mn{Dis}_{\Omc_S}(A' \sqcap \exists
  r^n . A) \sqsubseteq \exists r . M$. Consequently, the
  CIs~($\dagger$)
  are not contained in the approximation  $\Omc^\omega_T$ according
  to Figure~\ref{fig:thirdb}. It is compensated by the CIs
  $$
  \begin{array}{rcl}
      A &\sqsubseteq& X_{B_1 \sqcup B_2} \\[1mm]
      \exists r. X_{B_1 \sqcup B_2} &\sqsubseteq& X_{(B_1 \sqcap \exists
        r . B_1) \sqcup (B_2 \sqcap \exists r . B_2)} \\[1mm]
     A' \sqcap  X_{(B_1 \sqcap \exists
        r . B_1) \sqcup (B_2 \sqcap \exists r . B_2)} &\sqsubseteq&
      X_{\mn{Dis}_{\Omc_S}(A' \sqcap \exists r^n . A)} \\[1mm]
      X_{\mn{Dis}_{\Omc_S}(A' \sqcap \exists r^n . A)} &\sqsubseteq&
      \exists r . M
  \end{array}
  $$
  with the last line being an instantiation of the new second last CI
  schema in Figure~\ref{fig:thirdb}.
\end{exmp}
Proposition~\ref{ex:2} and Example~\ref{ex:ELUbot} provide cases where
the last line of Figure~\ref{fig:thirdb} is still needed. Arguably,
the cases illustrated by these examples are not too likely to occur in
practice.

\smallskip

It should be clear that the new CIs in the second last line are sound
and thus soundness of the approximation is not compromised.  In what
follows, we proof completeness.  For our proof to go through, we need
to assume that $\top$ is always contained in $\mn{sub}^-(\Omc_S)$.  We
start with observing two technical lemmas, the first one being a
variant of Lemma~\ref{lem:compaux}.
\begin{lemma}
\label{lem:modcompaux}
  Let $C_0$ be an \EL concept over $\mn{sig}(\Omc_S)$ decorated with
  disjunctions from $\mn{Dis}(\Omc_S)$ at leaves such that
  $\mn{Dis}_{\Omc_S}(C_0)$ has at least two disjuncts.
%
  Then 
  $\Omc^-_T \models C_0^\uparrow
  \sqsubseteq \mn{Dis}_{\Omc_S}(C_0)$.
\end{lemma}
\noindent
\begin{proof}\ (sketch) The proof is almost identical to that of
  Lemma~\ref{lem:compaux}, we only sketch the differences. The fact
  that $C_0$ is no longer an \EL concept but is decorated with
  disjunctions from $\mn{Dis}(\Omc_S)$ at leaves is no problem at
  all. It is simply carried through the entire proof and does not
  prompt any further modifications. The fact that we work with
  $\mn{Dis}_{\Omc_S}(C_0)$ instead of $\mn{Dis}^\EL_{\Omc_S}(C_0)$,
  however, does require some changes.  In the main proof of
  Lemma~\ref{lem:compaux}, we need a very slight modification of the
  special chase plus an adapted formulation of
  Lemma~\ref{lem:chasecompl}. 

  We define a variant $\vdash'$ of $\vdash$ that only differs in the
  clause for disjunction:
  \begin{itemize}

  \item $\Amc \vdash' C_1 \sqcup C_2(a)$ if (a)~$\Amc \models C_1(a)$ or
    (b)~$\Amc \models C_2(a)$ or (c)~$C_1 \sqcup C_2(a) \in \Amc$.

  \end{itemize}
  Now, the only modification of the special chase is that, in Rule~4,
  we replace $\Amc \vdash D_2(a)$ with $\Amc \vdash' D_2(a)$.  The
  adapted formulation of Lemma~\ref{lem:chasecompl} then reads as
  follows.
  \\[2mm]
  {\bf Claim 1.}  Let $\Omc$ be an $\ELU_\bot$ ontology and $C_0$ be an
  \EL concept over $\mn{sig}(\Omc)$ decorated with disjunctions from
  $\mn{Dis}(\Omc_S)$ at leaves such that
  $\mn{Dis}_{\Omc_S}(C_0)$ has at least two disjuncts. Then
  $X_{\mn{Dis}_\Omc(C_0)}(a_0)  \in \mn{ch}^{\text{sp}}_{\Omc}(\Amc_{C_0})$.
  \\[2mm]
  In the proof of Lemma~\ref{lem:compaux}, in the claim stating that
  $\Omc_T^- \models C^\uparrow_i \sqsubseteq C^\uparrow_{i+1}$ for
  all $i \geq 0$, we need to adapt the Case of Rule~4, as follows.

  \medskip

  Then there are $D_1(a) \in \Amc_i$ with $D_1 \in \mn{Dis}^-(\Omc_S)$
  and $D_2,D_3 \in \mn{Dis}(\Omc_S)$ such that $\Amc_i \vdash'
  D_2(a)$, $\Omc_S \models D_1 \sqcap D_2 \sqsubseteq D_3$, and
  $\Amc_{i+1}=\Amc_i \cup \{ D_3(a) \}$. Let $E_a$ be the subconcept
  of $C_i$ that corresponds to the subtree rooted at $a$ in $\Amc_i$
  and let $F_a$ be the subconcept of $C_{i+1}$ that corresponds to the
  subtree rooted at $a$ in $\Amc_{i+1}$. Then $F_a = E_a \sqcap D_3$.
  From $D_1(a) \in \Amc$ and $D_1 \in \mn{Dis}^-(\Omc_S)$, we obtain
  that $X_{D_1}$ is a top-level conjunct of $E_a^\uparrow$. From
  $\Amc_i \vdash' D_2(a)$, we obtain an \EL concept $D_2'$ with
  $\emptyset \models E_a^\uparrow \sqsubseteq {D_2'}^\uparrow$ and
  $\Omc_S \models D_2' \sqsubseteq D_2$; we in fact obtain $D_2$ by
  `following' $\Amc_i \vdash' D_2(a)$ using the definition of
  $\vdash'$ and whenever we arrive at $\Amc_i \vdash' F_1 \sqcup
  F_2(b)$ and this holds because of Case~(a) from the definition of
  $\vdash'$ (resp.\ Case~(b)), replacing the occurrence of $F_1 \sqcup
  F_2$ in $D_2$ that gave rise to this with $F_1$ (resp.\ $F_2$).
  From $\Omc_S \models D_1 \sqcap D_2 \sqsubseteq D_3$ and $\Omc_S
  \models D'_2 \sqsubseteq D_2$, we obtain $\Omc_S \models D_1 \sqcap
  D'_2 \sqsubseteq D_3$ and thus $\Omc_T^-$ contains the CI $X_{D_1}
  \sqcap {D'_2}^\uparrow \sqsubseteq D_3^\uparrow$. Consequently, $\Omc_T^-
  \models C^\uparrow_i \sqsubseteq C^\uparrow_{i+1}$ as required.

  \medskip

  It remains to prove Claim~1. The proof is, in turn, a slight
  modification of the proof of Lemma~\ref{lem:chasecompl}. Again, we
  concentrate on sketching the differences. Of course, we replace
  $\mn{Dis}^\EL_{\Omc}(C_0)$ with $\mn{Dis}_{\Omc}(C_0)$ throughout
  the proof. Further, we replace $\vdash$ with $\vdash'$ in property
  (P1) and in (the two incarnations of) property (P3). We then go on
  to construct the interpretation $\Imc$ as before and show, also as
  before, that it is a model of $\Omc$.  It remains to show that $a_0
  \notin \mn{Dis}_\Omc(C_0)$. For this, we first need to observe the
  following version of Claim~1 in the proof of
  Lemma~\ref{lem:chasecompl}.
  \\[2mm]
  {\bf Claim 2.} Let $a \in \Delta^\Imc$ be an individual of
  $\mn{ch}^{\text{sp}}_{\Omc}(\Amc_{C_0})$ and let
  $C \in \mn{sub}(\Omc)$ (not necessarily be an \EL concept). Then
$a \in C^\Imc$ implies 
  \begin{enumerate}

  \item $\emptyset \models E_a \sqsubseteq C$
    if $a$ is original and disjunctive, and

  \item $\mn{ch}^{\text{sp}}_{\Omc}(\Amc_{C_0}) \vdash' C(a)$ otherwise.

    \end{enumerate}
    The proof is by induction on the structure of $C$. All cases
    except $C= C_1 \sqcup C_2$ are as in the proof of Claim~1 in
    the proof of Lemma~\ref{lem:chasecompl}. Due to the use of
    $\vdash'$ in place of $\vdash$, however, the additional case
    is straightforward using the semantics and induction hypothesis.

\medskip

We next argue that $a_0$ is disjunctive. Assume to the contrary that
it is not.  It can be verified that Lemma~\ref{lem:chsound} (soundness
of the special chase) still holds when the precondition
$\mn{ch}^{\text{sp}}_{\Omc}(\Amc_{C_0}) \vdash D(a_0)$ is replaced
with $\mn{ch}^{\text{sp}}_{\Omc}(\Amc_{C_0}) \vdash' D(a_0)$.  Let $K$
be the conjunction of all $C \in \mn{sub}^-(\Omc)$ such that $a_0 \in
C^\Imc$. By Claim~2, $\mn{ch}^{\text{sp}}_{\Omc}(\Amc_{C_0}) \vdash'
C(a_0)$ for all such~$C$. Thus the modified Lemma~\ref{lem:chsound}
yields $\Omc \models C_0 \sqsubseteq K$. Since \Imc is a model of
\Omc, this implies that $\mn{Dis}_\Omc(C_0)$ has only the disjunct
$K$, a contradiction to  $\mn{Dis}_\Omc(C_0)$ having two disjuncts.

\smallskip Now back to 
our proof that $a_0 \notin \mn{Dis}_\Omc(C_0)$.
It remains to show that $a_0 \notin \mn{Dis}_\Omc(C_0)^\Imc$.  Then
there is a disjunct $K$ of $\mn{Dis}_\Omc(C_0)^\Imc$ such that $a_0
\in C^\Imc$ for every conjunct $C$ of~$K$. Since $a_0$ is disjunctive,
Point~1 of Claim~1, yields $\emptyset \models E_{a_0} \sqsubseteq C$
for all conjuncts~$C$ of~$K$.  Thus $\emptyset \models E_{a_0}
\sqsubseteq \mn{Dis}_\Omc(C_0)$, in contradiction to our choice of
$E_{a_0}$.
\end{proof}

\begin{lemma}
\label{lem:semadd}
  Let $D \in \mn{Dis}(\Omc_S)$ be satisfiable w.r.t.\ $\Omc_S$ and let
  $C$ be an \EL concept. Then
  $\Omc_S \models D \sqsubseteq \exists r . C$,
implies that there is a
  $D' \in \mn{Dis}(\Omc_S)$ with
  $\Omc_S \models D \sqsubseteq \exists r . D'$ and
  $\Omc_S \models D' \sqsubseteq C$.
\end{lemma}
\noindent
\begin{proof}\ For an interpretation \Imc and $d \in \Delta^\Imc$, let
  $\mn{Con}(d)$ denote the conjunction $K \in \mn{Con}(\Omc_S)$ such
  that for all $C \in \mn{sub}^-(\Omc_S)$, $d \in C^\Imc$ iff $C$ is a
  conjunct of $K$. Now consider all models $\Imc$ of $\Omc_S$ and all
  $d \in \Delta^\Imc$ with $d \in D^\Imc$. We use $\Kmc_{\Imc,d}$ to
  denote the set of all $K \in \mn{Con}(\Omc_S)$ such that
  $K=\mn{Con}(e)$ for some $r$-successor $e$ of $d$ in \Imc. Further,
  we use \Kmf to denote the set of all $\Kmc_{\Imc,d}$.
  \\[2mm]
  {\bf Claim.} For every $\Kmc_{\Imc,d} \in \Kmf$, there is a $K \in 
  \Kmc_{\Imc,d}$ with $\Omc_S \models K \sqsubseteq C$.
  \\[2mm]
  Assume that this is not the case. Then for each $K \in
  \Kmc_{\Imc,d}$ take a tree model $\Jmc_K$ of $\Omc_S$ with root
  $e_K$ such that $e_k \in K^{\Jmc_K} \setminus C^{\Jmc_K}$.  Then let
  the interpretation \Jmc be obtained from the unraveling of \Imc at
  $d$ by dropping all subtrees rooted at $r$-successors of the root
  $d$, taking the disjoint union with all $\Jmc_K$ and making each
  $e_K$ an $r$-successor of $d$. It can be verified that the resulting
  \Jmc is a model of $\Omc_S$ and that $d \in D^\Jmc \setminus
  (\exists r. C)^\Jmc$, in contradiction to $\Omc_S \models D
  \sqsubseteq \exists r . C$. This finishes the proof of the claim.

  \smallskip
  
  Now let $D'$ be the disjunction of all $K \in \Kmc_{\Imc,d}$ with
  $\Omc_S 
  \models K \sqsubseteq C$, over all  $\Kmc_{\Imc,d} \in \Kmf$. By the claim,
  $\Omc_S \models D' \sqsubseteq C$. Moreover, by definition of \Kmf,
  we have $\Omc_S \models D \sqsubseteq \exists r . D'$ and are done.
\end{proof}

Now back to the completeness proof of the modified approximation shown
in Figure~\ref{fig:thirdb}. Due to Lemma~\ref{lem:compbot}, it
suffices to show that for all CIs $F^\uparrow \sqsubseteq \exists r
. G$ with $F$ and $G$ of the form required for the last line of
Figure~\ref{fig:thirdb} and Property~2 from Figure~\ref{fig:thirdb}
not satisfied, then the restriction $\Omc^*_T$ of $\Omc^\omega_T$ to
the first four lines is such that $\Omc^*_T \models F^\uparrow
\sqsubseteq \exists r . G$.

\smallskip

Thus take a CI $F^\uparrow \sqsubseteq \exists r . G$ as described.
Then $D=\mn{Dis}_{\Omc_S}(F)$ has more than one disjunct and $\Omc_S
\models \mn{Dis}_{\Omc_S}(F) \sqsubseteq \exists r . G$. By
Lemma~\ref{lem:modcompaux}, $\Omc^-_T \models F^\uparrow \sqsubseteq
X_{\mn{Dis}_{\Omc_S}(F)}$. Moreover, $\mn{Dis}_{\Omc_S}(F)$ is
satisfiable w.r.t.\ $\Omc_S$ since it contains at least two
disjuncts. 
To show that $\Omc^*_T \models F^\uparrow
\sqsubseteq \exists r . G$, it thus suffices to establish the
following.
\\[2mm]
{\bf Claim.} If $\Omc_S \models D \sqsubseteq C$ with $D \in
\mn{Dis}^-(\Omc_S)$ satisfiable w.r.t.\ $\Omc_S$ and $C$ an \EL
concept, then $\Omc^*_T \models X_D \sqsubseteq C$.
\\[2mm]
We prove the claim by induction on $C$. If $C=A$ is a concept name,
then it follows from $\Omc_S \models D \sqsubseteq C$ that $\Omc^-_T$
contains a CI $X_D \sqcap X_D \sqsubseteq A$, and thus we are done.
The case that $C=C_1 \sqcap C_2$ is straightforward using the
semantics and induction hypothesis. Thus assume that
$C=\exists r . C_1$. By Lemma~\ref{lem:semadd}, there is a
$D' \in \mn{Dis}(\Omc_S)$ with
$\Omc_S \models D \sqsubseteq \exists r . D'$ and
$\Omc_S \models D' \sqsubseteq C_1$. We can find a disjunction $D''$
with at least two disjuncts such that $\Omc_S \models D' \equiv D''$:
if $D'$ has only a single disjunct that does not contain $\top$ as a
conjunct, we can choose $D''=D' \sqcup (D' \sqcap \top)$ and if $D'$
has only a single disjunct that does contain $\top$ as a conjunct, we
can choose $D''=D' \sqcup D^-$ where $D^-$ is $D'$ with conjunct
$\top$ removed. We can apply the induction hypothesis to $D''$ and
$C_1$ to obtain $\Omc^*_T \models X_{D''} \sqsubseteq C_1$. Moreover,
by the second last line in Figure~\ref{fig:thirdb}, $\Omc^*_T$ contains
$X_D \sqsubseteq \exists r . D''$ and thus we have $\Omc^*_T
\models X_{D} \sqsubseteq \exists r . C_1$, as required.

\thmcorrminus*
\noindent
\begin{proof}\ The `if' direction follows from
  Lemma~\ref{lem:chasebotcorr}. For `only if', assume that $\Omc_S \models
  C_0 \sqsubseteq D_0$. By Lemma~\ref{lem:compaux}, $\Omc_T^- \models
  C_0 \sqsubseteq \mn{Dis}^\EL_{\Omc_S}(C_0)^\uparrow$. By definition
  of $\mn{Dis}^\EL_{\Omc_S}(C_0)$, $\Omc_S \models C_0 \sqsubseteq
  D_0$ and $D_0 \in \mn{sub}(\Omc_S)$ implies that every top-level
  conjunct of $D_0$ is a conjunct in every disjunct of
  $\mn{Dis}^\EL_{\Omc_S}(C_0)$. First assume that there is only a
  single such disjunct. Then $\mn{Dis}^\EL_{\Omc_S}(C_0)$ with
  conjunct $D_0$, and since $D_0$ is an \EL concept it is also a
  conjunct of $\mn{Dis}^\EL_{\Omc_S}(C_0)^\uparrow$. Thus $\Omc_T^-
  \models C_0 \sqsubseteq \mn{Dis}^\EL_{\Omc_S}(C_0)^\uparrow$ implies
  $\Omc^-_T \models C_0 \sqsubseteq D_0$ as required. Now assume that
  $\mn{Dis}^\EL_{\Omc_S}(C_0)$ has more than one disjunct. Then
  $\mn{Dis}^\EL_{\Omc_S}(C_0)^\uparrow=X_{\mn{Dis}^\EL_{\Omc_S}(C_0)}$
  and $\Omc^-_T$ contains the CI $X_{\mn{Dis}^\EL_{\Omc_S}(C_0)}
  \sqcap X_{\mn{Dis}^\EL_{\Omc_S}(C_0)} \sqsubseteq D_0$. Thus again
  $\Omc^-_T \models C_0 \sqsubseteq D_0$.
\end{proof}

\section{Proof of Theorem~\ref{thm:OMQresults}}

\thmOMQresults*

\noindent
\begin{proof} \
Define the unfolding $\Amc^{\ast}_{a}$ of an ABox $\Amc$ at an individual names $a$ as 
the (possibly infinite) ABox whose individuals are words $w$ of the form 
$a_{0}r_{1}a_{1}\ldots a_{n}$ with $a_{0}=a$ and $r_{i+1}(a_{i},a_{i+1})\in \Amc$ for all $i<n$, 
and containing the assertions $A(a_{0}r_{1}a_{1}\ldots a_{n})$ if $A(a_{n})\in \Amc$ and
$r(a_{0}r_{1}a_{1}\ldots a_{n},a_{0}r_{1}\ldots a_{n}r_{n+1}a_{n+1})$ if $r_{n+1}(a_{n},a_{n+1})\in \Amc$.
The following has been proved in \cite{DBLP:journals/jsc/LutzW10}.

\medskip
\noindent
\emph{Fact 1.} The following conditions are equivalent for any $\EL_{\bot}$ ontology $\Omc$ and \EL concept $C$:
\begin{enumerate}
\item $\Omc,\Amc\models C(a)$;
\item $\Omc,\Amc^{\ast}_{a}\models C(a)$.
\end{enumerate}
We now show the first claim of Theorem~\ref{thm:OMQresults}. 
The proofs of the remaining two claims are similar and omitted.
Let $\Omc_{S}$ be an \ALC ontology with $\mn{sig}(\Omc_{S})=\Sigma$
and let $\Omc_{T}^{\omega}$ be the ontology from Section~4. To show that $\Omc_{T}^{\omega}$
is an $\EL_{\bot}$ approximation of $\Omc_{S}$ w.r.t.~ELQ$(\Sigma)$, we have to check
the conditions of Definition~\ref{def:approx_queryNEW}. For Condition~1, assume that $C(x)$ is
in ELQ$(\Sigma)$ and that $\Amc$ is an ABox using no symbols from 
$\mn{sig}(\Omc_{T}^{\omega})\setminus \mn{sig}(\Omc_{S})$ such that $\Omc_{T}^{\omega},\Amc\models C(a)$.
By Fact~1, $\Omc_{T}^{\omega},\Amc_{a}^{\ast}\models C(a)$. Denote by $(\Amc_{a}^{\ast})_{|\Sigma}$
the ABox obtained from $\Amc_{a}^{\ast}$ by removing all assertions using symbols not in $\Sigma$.
Then still $\Omc_{T}^{\omega},(\Amc_{a}^{\ast})_{|\Sigma}\models C(a)$ as $\Omc_{T}^{\omega}$ and $C$ do not
use any of the symbols used in the assertions we removed. By compactness there exists 
an \EL concept $D$ corresponding to a finite subABox $\Amc_{1}$ of $(\Amc_{a}^{\ast})_{|\Sigma}$
with root $a$ such that $\Omc_{T}^{\omega}\models D \sqsubseteq C$. Then $\Omc_{S}\models D \sqsubseteq C$
since $\Omc_{T}^{\omega}$ is an $\EL_{\bot}$ approximation of $\Omc_{S}$ and $C,D$ use symbols in $\Sigma$
only. Then $\Omc_{S},\Amc\models C(a)$
since there is a homomorphism from $\Amc_{1}$ to $\Amc$ mapping $a$ to $a$.

For Condition~2, let $C(x)$ be in ELQ$(\Sigma)$ and $Q=(\Omc,\Sigma',C(x))$ such that
$(\Omc_{S},\Sigma',C(x))\supseteq Q$, where $\Sigma'$ is a signature with 
$\Sigma'\cap \mn{sig}(\Omc_{T}^{\omega}) \subseteq \mn{sig}(\Omc_{S})$ and $\Omc$ is an $\EL_{\bot}$ ontology.
To show that $(\Omc_{T}^{\omega},\Sigma',C(x))\supseteq Q$, consider a $\Sigma'$ ABox $\Amc$
such that $\Omc,\Amc\models C(a)$. Then by Fact~1, $\Omc,\Amc_{a}^{\ast}\models C(a)$.
Hence $\Omc_{S},\Amc_{a}^{\ast}\models C(a)$ since $(\Omc_{S},\Sigma',C(x))\supseteq Q$.
Then one can argue as above that $\Omc_{T}^{\omega},\Amc_{a}^{\ast}\models C(a)$.
Hence, by Fact~1, $\Omc_{T}^{\omega},\Amc\models C(a)$, as required.
\end{proof}

\end{document}